\documentclass[preprint,12pt,3p]{elsarticle}

\usepackage{amssymb}

\usepackage{lineno,hyperref}
\usepackage{amsmath}
\usepackage{amsfonts}
\usepackage{algorithm}
\usepackage{algorithmic}
\usepackage{array}
\usepackage{graphicx}
\usepackage{subfigure}
\usepackage{booktabs}
\usepackage{color}
\usepackage{multirow}

\usepackage{newfloat}
\usepackage{listings}

\newcommand{\argmin}{\operatornamewithlimits{arg\,min}}
\newcommand{\BlackBox}{\rule{1.5ex}{1.5ex}}  
\newenvironment{proof}{\par\noindent{\bf Proof\ }}{\hfill\BlackBox\\[2mm]}
\newtheorem{assumption}{Assumption}
 
\newtheorem{theorem}{Theorem}
\newtheorem{lemma}{Lemma} 
\newtheorem{proposition}{Proposition} 
\newtheorem{remark}{Remark}
\newtheorem{corollary}{Corollary}
\newtheorem{definition}{Definition}

\def\SS{\boldsymbol S}
\def\WW{\boldsymbol W}
\def\XX{\boldsymbol X}
\def\OO{\boldsymbol \Omega}
\def\oo{\boldsymbol \omega}
\def\xx{\boldsymbol x}
\def\yy{\boldsymbol y}

\def\ee{\boldsymbol \epsilon}

\journal{Pattern Recognition}

\begin{document}

\begin{frontmatter}

\author[1]{Jian Li\corref{cor1}}
\ead{lijian9026@iie.ac.cn}

\author[2]{Yong Liu}
\ead{liuyonggsai@ruc.edu.cn}

\author[1]{Weiping Wang}
\ead{wangweiping@iie.ac.cn}

\cortext[cor1]{Corresponding author}

\address[1]{Institute of Information Engineering, Chinese Academy of Sciences.}

\address[2]{Gaoling School of Artificial Intelligence, Renmin University of China.}

\title{Semi-supervised Vector-valued Learning: \\Improved Bounds and Algorithms}

\begin{abstract}
  Vector-valued learning, where the output space admits a vector-valued structure, is an important problem that covers a broad family of important domains, e.g. multi-task learning and transfer learning. Using local Rademacher complexity and unlabeled data, we derive novel semi-supervised excess risk bounds for general vector-valued learning from both kernel perspective and linear perspective. The derived bounds are much sharper than existing ones and the convergence rates are improved from the square root of labeled sample size to the square root of total sample size or directly dependent on labeled sample size. Motivated by our theoretical analysis, we propose a general semi-supervised algorithm for efficiently learning vector-valued functions, incorporating both local Rademacher complexity and Laplacian regularization. Extensive experimental results illustrate the proposed algorithm significantly outperforms the compared methods, which coincides with our theoretical findings.
\end{abstract}


\end{frontmatter}

\section{Introduction}
\label{sec:introduction}

Learning vector-valued functions involves learning a predictive model from training data that has vector-valued rather than scalar-valued labels. This encompasses a wide range of important tasks, such as multi-class classification \cite{torralba2007sharing,cortes2013multi}, multi-label learning \cite{zhang2014lift,qian2022weight}, multi-task learning \cite{ji2013multitask}, transfer learning \cite{ye2021implementing} and so on.

On the algorithmic front, various models have been developed to address special cases of vector-valued learning \cite{tschumperle2005vector}, with a particular focus on multi-class classification and multi-label learning. 
Multi-class classification is a standard learning paradigm classifying instances into one of more than two classes. Conventional multi-class classification approaches include error-correcting output codes (ECOC) \cite{liu2022novel}, multiclass support vector machine (multiclass SVM) \cite{hsu2002comparison,cortes2013multi,chen2022online}, and more algorithms refer to \cite{duan2021oaa}.
Meanwhile, multi-label learning assigns multiple labels for each instance simultaneously, including multi-label learning with missing labels \cite{yu2014large}, semi-supervised multi-label learning \cite{mojoo2017deep},  extreme multi-label classification \cite{wydmuch2018no}, partial multi-label learning \cite{xie2021partial}, and more examples in \cite{liu2021emerging}.
For vector-valued functions which admit a reproducing kernel, 
\cite{carmeli2010vector} presented an algorithm to learn the reproducing kernel Hilbert space (RKHS), and then \cite{quang2013unifying} extended it to semi-supervised learning via manifold regularization.
However, the learning framework for general vector-valued tasks has been scarcely studied.

On the theoretical front, the statistical learning theory for vector-valued functions suggests that estimating the generalization ability of algorithms is key to understanding the factors that affect their performance and developing ways to improve them \cite{bartlett2002rademacher}.
The statistical learning theories for special cases of vector-valued functions (e.g. multi-class classification and multi-label learning) have, by now, been well-developed \cite{cortes2013multi,yu2014large,lei2019data}. 
For instance, the convergence rates of the generalization error bounds for multi-class classification and multi-label learning are $\mathcal{O}(K/\sqrt{n})$ and $\mathcal{O}(1/\sqrt{n})$, respectively, where $K$ is the size of the vector-valued output and $n$ is the number of labeled samples.
However, despite its importance, theoretical properties for vector-valued functions have been only scarcely studied.
Theoretical results from recent works \cite{cortes2016structured,maurer2016vector,wu2021fine} on vector-valued functions employed the contraction inequality to estimate the global Rademacher complexity of the estimators,
while the rates of their bounds are at best $\mathcal{O}(1/\sqrt{n}).$

As shown in Figure \ref{fig.structure}, our main contributions lie in both theory and algorithm.
In this paper, we first define the general schema for vector-valued learning $f: \mathbb{R}^d \to \mathbb{R}^K$ by a linear estimator $f(\xx)=\WW^T\phi(\xx)$ with a nonlinear feature mapping $\phi: \mathbb{R}^d \to \mathcal{S}$.
The general schema includes a wide range of machine learning algorithms, such as kernel methods, (deep) neural networks, random features, generalized linear models and so on.
Then, we derive novel data-dependent generalization error bounds by making use of local Rademacher complexity and unlabeled data for vector-valued learning.
A unified learning framework is further designed and solved by proximal gradient descent.
Extensive experiments verify the effectiveness of the algorithm and support the statistical findings.

\textbf{Theoretical contributions.}
We provide theoretical guarantees for learning vector-valued functions in both the kernel space and linear space.
Instead of global Rademacher complexity, we exploit \textit{local} Rademacher complexity to improve the convergence rate of excess risk bounds from $\mathcal{O}(1/\sqrt{n})$ to $\mathcal{O}(1/n)$, where $n$ is the number of labeled examples.
Unlabeled samples are used to reduce the estimate of Rademacher complexity and the learning rate is improved from $\mathcal{O}(1/\sqrt{n})$ to $\mathcal{O}(1/\sqrt{n+u} + 1/n)$, where $u$ is the number of unlabeled examples. 
Finally, we establish the unified excess risk bounds, for which we obtain, for the first time, data-dependent error bounds for semi-supervised vector-valued learning from both the kernel perspective and linear perspective.
To the best of our knowledge, these excess risk bounds are the tightest bounds developed so far for vector-valued learning and can be applied to various vector-valued tasks.

\textbf{Algorithmic contributions.}
Motivated by our theoretical analysis, we devise a unified learning framework for vector-valued functions, covering multi-class classification, multi-label learning, multi-task learning, transfer learning and so on.
Meanwhile, the nonlinear feature mapping in the algorithm is very general, which can be the kernel, the neural network, random features and linear transformation.
The framework combines the empirical risk minimization (ERM) framework with two additional terms to bound local Rademacher complexity and makes use of unlabeled samples.
The tail sum of singular values of the weight matrix in the predictor is used to bound local Rademacher complexity.
Under semi-supervised settings, Laplacian regularization is introduced to make use of unlabeled samples, making the algorithm suitable for these settings.
Using proximal gradient descent and nonlinear feature mappings, the algorithm effectively solves the non-differentiable optimization problem on the primal, achieving a good tradeoff between efficiency and accuracy.

The rest of the paper is organized as follows.
We begin with the related work and the improvements on our previous papers in Section \ref{sec.related_work}.
Some preliminaries for vector-valued learning are introduced in Section \ref{sec.preliminaries}.
Then, we derive excess risk bounds for vector-valued learning in both the kernel and linear spaces in Section \ref{sec.theory}, and compare our bounds with existing vector-valued learning bounds in Section \ref{sec.theory_comparison}.
In Section \ref{sec.algorithm}, we present a unified learning framework for vector-valued learning.
We empirically validate the theoretical findings and algorithm in Section \ref{sec.experiment}.
Finally, we concluded this work in Section \ref{sec.conclusion}.

\section{Related Work}
\label{sec.related_work}
  In this section, we introduce the generalization theories of vector-valued learning and the improvements in our previous work.
  Recently, using the integral operator theory, \cite{ciliberto2020general} proved consistency and excess risk bounds for a general framework for structured prediction, and then \cite{brogat2022vector} further proposed generalization bounds for vector-valued least square regression.
  Meanwhile, \cite{li2021towards,wu2021fine} derived shaper generalization error bounds for structured prediction.
  Besides, there are also theoretical bounds for multi-view semi-supervised learning with Laplacian regularization \cite{sun2010sparse,sun2011multi} using Rademacher complexity. 
  
\subsection{Statistical Properties of Multi-class Classification}
  The generalization ability of multi-class classification has been analyzed by many papers.
  For example, \cite{daniely2015multiclass} used Natarajan dimension to tightly characterize the sample complexity of multiclass learning in the PAC setting, achieve the convergence rate $\mathcal{O}(\sqrt{d_N/n})$, where $d_N$ is the Natarajan dimension and $n$ is the number of labeled examples.
  Data-dependent complexity tools, on the other hand, always yield tighter bounds.
  As the most common and successful data-dependent measure, Rademacher complexity was first used to analyze the generalization ability of multiclass SVM in \cite{koltchinskii2001some,cortes2013multi}, of which the convergence rates are usually $\mathcal{O}(K/\sqrt{n}),$ where $K$ is the number of classes.
  By bridging Gaussian complexity and Rademacher complexity, \cite{lei2019data} devised a generalization error bound which exhibits logarithmic dependence on the class size $\mathcal{O}((\log K)/\sqrt{n})$.
  Under the relatively strict assumption that the derivative of the loss function is $L$-Lipschitz continuous, our previous work \cite{li2018multi} proposed the state-of-the-art error bound for kernel-based multi-class classification using local Rademacher complexity, where the convergence rate is inversely proportional to the number of samples $n.$
  Combining Rademacher complexity with the use of unlabeled data, Maximov et al. proposed a semi-supervised multi-class bound \cite{maximov2018rademacher}, which has a convergence rate of $\mathcal{O}\big(\sqrt{{K}/{n}}+K\sqrt{{K}/{u}}\big),$ where $u$ is the number of unlabeled samples.
  Further, based on local Rademacher complexity, we extended our previous work \cite{li2018multi} to semi-supervised settings in \cite{li2019multi} and obtained a tight semi-supervised bound, of which the rate is $\mathcal{O}\big({K}/{\sqrt{n+u}} + {1}/{n}\big).$
  Beyond kernel methods, there are a line of work on deep neural networks for multiclass classification that proved gradient descent on cross-entropy may converge to the max-margin solution with zero training loss \cite{lyu2021gradient}, while \cite{wang2021benign} proven several multiclass algorithms lead to interpolation and their equivalence.
  
\subsection{Statistical Properties of Multi-label Learning}
  The consistency of multi-label learning, whether the expected loss converges to the Bayes loss as the number of training examples increases, has been studied in \cite{gao2011consistency} for ranking loss and hamming loss.
  Recently, \cite{wydmuch2018no} proved the consistency of probabilistic label trees (PLTs) for extreme multi-label classification.
  The generalization bounds are mainly studied based on the Rademacher complexity.
  For example, using the global Rademacher complexity, \cite{yu2014large} derived the generalization error bounds for low-rank linear multi-label model in the standard empirical minimization error (ERM) framework where the convergence rate is $\mathcal{O}(1/\sqrt{n}),$ and then it demonstrated the superiority of low-rank promoting trace-norm regularization over Frobenius regularization.
  \cite{xu2016local} proposed to minimize the tail sum of the singular values of the predictor in multi-label learning and obtain a faster convergence rate $\mathcal{O}(1/n)$ using the local Rademacher complexity.
  Both the theoretical analysis and proposed algorithms of \cite{yu2014large,xu2016local} are in the linear space,
  while this paper explores statistical properties and designs a unified algorithm for both the kernel space and linear space.
  
  \begin{figure}[t]
    \begin{center}
      \includegraphics[width=0.7\linewidth]{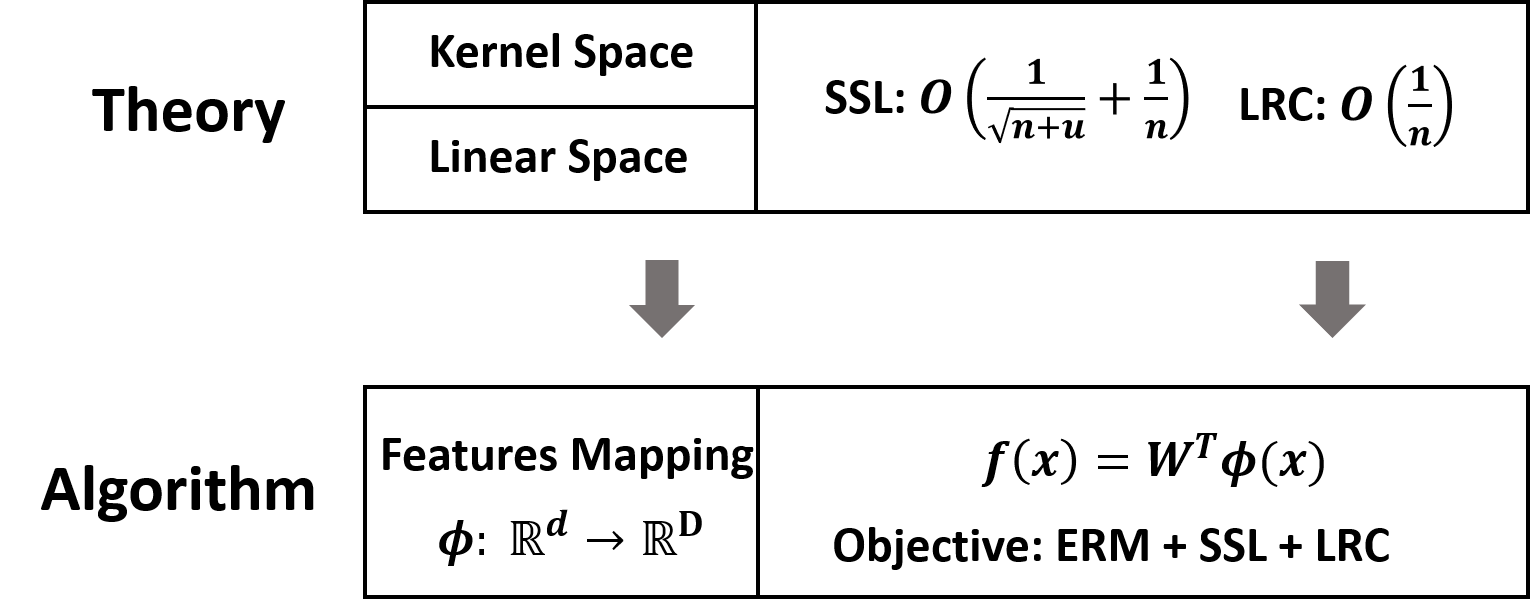}
    \end{center}
    \caption{Main contributions of this paper. 
    We derive excess risk bounds for vector-valued learning in both kernel space and linear space, which achieve the learning rate $\mathcal{O}(1/\sqrt{n+u} + 1/n)$.
    Driven by the theoretical findings, we propose an efficient semi-supervised vector-valued learning algorithm with nonlinear feature mappings and the local Rademacher complexity Regularization term.
    }
    \label{fig.structure}
  \end{figure}

\subsection{Improvements on Our Previous Work}
  Our previous works \cite{li2018multi,li2019multi}, published in {NeurIPS} and {IJCAI}, provided origin ideas on the generalization analysis of multi-class classification with local Rademacher complexity.
  We then emphasize the contributions of this work in relation to the previous ones, by illustrating the shortcomings of previous work and major improvements in this work.
  
  \begin{table}[t]
    \caption{Novel contributions of this work. MC represents multi-class classification. SS-MC represents semi-supervised MC. SS-VV represents semi-supervised vector-valued learning. MKL-MC represents multiple kernel learning for MC.}
    \label{tab.contribution}
    \centering
    \small
    \begin{tabular}{l|llllll}
        \toprule
        Reference &Task   & Hypothesis   & Condition & Rate & Algorithm       \\ \hline
        \cite{li2018multi} &MC  &Kernel &Smoothness  & $\mathcal{O}(\log K / n)$ & MKL-MC \\
        \cite{li2019multi} &SS-MC  &Linear &Smoothness  & $\mathcal{O}(\frac{K}{\sqrt{n + u}} + \frac{1}{n})$ & Linear estimator \\
        This paper &SS-VV  &Linear \& Kernel  &Continuity  & $\mathcal{O}(\frac{1}{\sqrt{n + u}} + \frac{1}{n})$ & Approximate  kernel \\
        \bottomrule
    \end{tabular}
  \end{table}

  In the previous work \cite{li2018multi}, we improved the generalization error bounds for \textit{multi-class classification} in the kernel space, while \cite{li2019multi} provided the generalization analysis for semi-supervised multi-class classification in the linear space.
  However, previous literature shows several shortcomings in theoretical guarantees and the effectiveness of algorithms:
  1) {Theoretical shortcomings.} 
  The theoretical findings in \cite{li2018multi,li2019multi} only work for multi-class classification and fail to apply to more general vector-valued cases.
  Meanwhile, the proofs in \cite{li2018multi} for kernelized multi-class classification are rather complicated, while the proofs in \cite{li2019multi} is hard to apply to kernel settings.
  2) {Algorithmic shortcomings.}
  \cite{li2018multi} employed multiple kernel learning for multi-class classification (MKL-MC) that is inefficient for large scale tasks, while the linear classifier used in \cite{li2019multi} often leads to inferior performance.
  To overcome these drawbacks, as shown in Table \ref{tab.contribution}, we make the following significant improvements:
  
  \textbf{1) Theoretical guarantees for more general cases.} We provide an unified theoretical results for general semi-supervised vector-valued tasks (i,e, multi-label learning, multi-task learning, transfer learning and co-kriging) in both kernel space and linear space, while the previous literature only pertains for the MKL multi-class classifier \cite{li2018multi} or the linear multi-class classifier \cite{li2019multi}.
  
  \textbf{2) Milder condition.} We use a milder assumption for the loss function.  Our earlier work \cite{li2018multi} assumed the loss function to be $L$-smooth, while this paper just needs $L$-Lipschitz continuous condition on the loss function, covering much loss functions.

  \textbf{3) Concise proof details for kernel classifiers.} Using the contraction inequality (Lemma \ref{lem.contraction}) to directly bound local Rademacher complexity of hypothesis space, we simply the proof of sharper analysis, while \cite{li2018multi} employed a complex derivation from the Gaussian complexity to the local Rademacher complexity.
  
  \textbf{4) Good tradeoff between accuracy and efficiency.} 
  We proposed an efficient algorithm for vector-valued functions in approximate kernel space using the nonlinear feature mapping space.
  Specifically, the nonlinear feature mapping characterizes excellent predictive ability in the manner of linear classifiers, avoiding the low efficiency of multiple kernel learning (MKL-MC) \cite{li2018multi} and inferior performance of the linear classifier \cite{li2019multi}.

  \textbf{5) More experiments on vector-valued tasks.} 
  To validate our theoretical findings, we perform more experiments on valued-valued learning tasks.
  Beside multi-class classification tasks \cite{li2018multi,li2019multi}, multi-label learning datasets are also used in this paper.
  Additionally, we explore the influences of using local Rademacher complexity and unlabeled samples.

\section{Problem Setting and Preliminaries}
\label{sec.preliminaries}
We define vector-valued problems on input space $\mathcal{X} = \mathbb{R}^d$
and output space $\mathcal{Y},$ which produces vector-valued outputs $\mathcal{Y} \subseteq \mathbb{R}^K$ (such as multivariate labels).
We consider a set of labeled training samples $\mathcal{D}_l=\{(\xx_i, ~ \yy_i)\}_{i=1}^n$
i.i.d. drawn from some unknown distribution $\rho$ over $\mathcal{X} \times \mathcal{Y}$
and unlabeled samples $\mathcal{D}_u = \{\xx_i\}_{i=n + 1}^{n + u}$ i.i.d. sampled according to the marginal distribution $\rho_X$ of $\rho$ over $\mathcal{X}.$
Typically, only a few labeled samples and a large number of unlabeled samples are available, that is, $n \ll u.$

\subsection{The Vector-valued Learning Framework}
The goal is to learn a vector-valued estimator $h: \mathbb{R}^{d} \to \mathbb{R}^K,$ which outputs $K$-dimensional labels.
We define a general hypothesis space for both kernel-based and linear methods
\begin{align}
  \label{eq.hypothesis}
  \mathcal{H}_p = \left\{\xx \to h(\xx)=\WW^T\phi(\xx): \|\WW\|_p \leq 1 \right\},
\end{align}
where $\WW \in \mathcal{S} \times \mathbb{R}^{K}$ is the weight matrix, $\phi(\xx): \mathbb{R}^d \to \mathcal{S}$ is a feature mapping (linear or non-linear), $\mathcal{S}$ is the feature space and $\|\WW\|_p$ is a matrix norm to regularize the hypothesis.
Specifically, to analyze the generalization performance of vector-valued functions, we use the trace norm $\|\WW\|_* \leq 1$ in both the kernel space and the linear space.
Then, we present specific estimators for the linear space and kernel space.

We denote the loss function $\ell: \mathcal{Y} \times \mathcal{Y} \to \mathbb{R}_+$ to measure the dissimilarity between two elements from vector-valued outputs.
The target of statistical learning is to minimize the expected loss
\begin{align*}
 \mathcal{E}(h) = \int_{\mathcal{X} \times \mathcal{Y}} \ell(h(\xx), ~ \yy) d \rho(\xx, ~ \yy),
\end{align*}
where $\ell$ is the loss function and $h \in  \mathcal{H}_p.$
The empirical loss is usually defined as $\widehat{\mathcal{E}}(h)=\frac{1}{n}\sum_{i=1}^n ~ \ell(h(\xx_i), ~ \yy_i).$
For the sake of simplicity, we assume that the loss function is bounded $\ell: \mathcal{Y} \times \mathcal{Y} \to [0, B]$, where $B > 0$ is a constant.
This is a common restriction on the loss function, satisfied by the bounded hypothesis.
Moreover, we normalize the inner product $\langle \phi(\xx), \phi(\xx') \rangle \leq 1$, so we have $\sup_{\xx \in \mathcal{X}} \kappa(\xx,\xx) \leq 1$ for kernel estimators and $\mathbb{E} [\xx^T\xx] \leq 1$ for linear estimators.
The definition of hypothesis space \eqref{eq.hypothesis} involves several kinds of classifiers in terms of different feature mappings $\phi$:

\textbf{1) Kernel methods.}
Let $\kappa: \mathcal{X} \times \mathcal{X} \to \mathbb{R}$ be a Mercer kernel with the associated feature map $\phi$ and reproducing kernel Hilbert space $H_\kappa,$ where $\kappa(\xx, \xx') = \langle \phi(\xx), \phi(\xx')\rangle$ and $\phi: \mathbb{R}^d \to H_\kappa,$ thus $\mathcal{S} = H_\kappa.$

\textbf{2) Approximate kernel approaches.}
Random features (RF) \cite{rahimi2007random}, Nystr\"om approximation \cite{williams2001using} and randomized sketches \cite{yang2017randomized} can be seen as kernel approximation via $\kappa(\xx, \xx') \approx \langle \phi(\xx), \phi(\xx') \rangle,$ where $\phi: \mathbb{R}^d \to \mathbb{R}^D$ is an explicit feature mapping and $\mathcal{S} = \mathbb{R}^D.$

\textbf{3) (Deep) neural networks.} Supervised deep neural network can be divided into two parts: the nonlinear feature mappings $\phi: \mathbb{R}^d \to \mathbb{R}^D$ before the last layer and the linear estimator $h: \mathbb{R}^D \to \mathbb{R}^K$ with $h(\xx) = \WW^T\phi(\xx)$ in the last layer.
Assume there are $L$ layers in the neural network, and then the feature mapping can be written as $\phi(\cdot) = \phi_{L-1}(\phi_{L-2}(\cdots, \phi_1(\cdot) ,\cdots))$.
It can be seen as a deep kernel function via multiple layers.

\textbf{4) Linear classifiers.}
The commonly used linear estimators are directly in the input space $\phi(\xx)=\xx$ and $\mathcal{S} = \mathbb{R}^d.$

\subsection{Notations and Assumptions}
The space for the loss functions associated with $ \mathcal{H}_p$ is
\begin{align}
 \label{eq.loss-space}
 \mathcal{L}=\left\{\ell(h(\xx), ~ \yy) ~ \big | ~ h\in \mathcal{H}_p\right\}.
\end{align}
\begin{definition}[Rademacher complexity of the loss space]
 \label{def.rc-loss}
 Assume $\mathcal{L}$ is the space for loss functions defined in Equation \eqref{eq.loss-space}.
 Then the empirical Rademacher complexity of $\mathcal{L}$ on $\mathcal{D}_l$ is:
 \begin{align}
  \label{eq.grc-loss}
 &\widehat{\mathcal{R}}(\mathcal{L})=
 \frac{1}{n} ~ \mathbb{E}_{\epsilon}
 \left[\sup_{\ell\in\mathcal{L}}
 \sum_{i=1}^n\epsilon_i \ell (h(\xx_i), ~ \yy_i)
 \right],
 \end{align}
 where $\epsilon_i$s are random independent Rademacher variables uniformly distributed over $\{\pm1\}.$
 Its deterministic counterpart is $\mathcal{R}(\mathcal{L})=\mathbb{E} ~ \widehat{\mathcal{R}}(\mathcal{L}).$
\end{definition}

\begin{definition}[Local Rademacher complexity of loss space]
 \label{def.lrc-loss}
 For any $r>0,$ local Rademacher complexity of $\mathcal{L}$ is
 \begin{align}
  \label{eq.lrc-loss}
 \mathcal{R}(\mathcal{L}_r)=
 \mathcal{R}\left(\left\{
 \ell_h ~\big|~
 \ell_h \in \mathcal{L}, ~ \mathbb{E} ~ (\ell_h - \ell_{h^*})^2 \leq r
 \right\}\right),
 \end{align}
 where $\ell_{h^*}$ represents the minimal expected loss.
 \end{definition}
 From (\ref{eq.grc-loss}) to (\ref{eq.lrc-loss}), a smaller class $\mathcal{L}_r\subseteq\mathcal{L}$ is selected
 by a ball around the minimal expected loss $\ell_{h^*}$ with a fixed radius $r.$
 The corresponding localized hypothesis space is
 \begin{align}
 \label{eq.localized-hypothesis-space}
 \mathcal{H}_r = \{h ~ \big | ~ h \in  \mathcal{H}_p, ~ \mathbb{E} ~ (\ell_h - \ell_{h^*})^2 \leq r\}.
 \end{align}

 Definitions \ref{def.rc-loss} and \ref{def.lrc-loss} demonstrate that Rademacher complexity of the loss space is output-dependent, such that the empirical counterparts can be estimated only on the labeled samples $\mathcal{D}_l.$
 In the following definition, we introduce the notion of Rademacher complexity of the hypothesis space, which is output-independent and can be estimated on both the labeled samples $\mathcal{D}_l$ and unlabeled samples $\mathcal{D}_u.$
\begin{definition}[Local Rademacher complexity of hypothesis space]
 Assume that the localized hypothesis space $\mathcal{H}_r$ is defined as in (\ref{eq.localized-hypothesis-space}).
 The empirical local Rademacher complexity of $\mathcal{H}_r$ on both labeled and unlabeled samples $\mathcal{D}_l \cup \mathcal{D}_u$ is defined as:
 \begin{align*}
 \widehat{\mathcal{R}}(\mathcal{H}_r) =
 \frac{1}{n + u} ~ \mathbb{E}_{\epsilon}
 \left[\sup_{h \in\mathcal{H}_r}
 \sum_{i=1}^{n + u} \sum_{j=1}^K \ee_{ik} h_j(\xx_i)
 \right],
 \end{align*}
 where $h_j(\xx_i)$ is the $j$-th value in the vector-valued function $h(\xx_i)$ with $K$ outputs and $\ee_{ik}$s are $(n + u) \times K$ Rademacher variables.
 The deterministic counterpart is $\mathcal{R}(\mathcal{H}_r)=\mathbb{E} ~ \widehat{\mathcal{R}}(\mathcal{H}_r).$
\end{definition}

Consider the loss function $\ell$ is bounded $\ell \in [0, B]$ and satisfies the following condition.

\begin{assumption}
  \label{ass.lipschitz}
  We assume that the loss function $\ell: \mathcal{Y} \times \mathcal{Y} \to \mathbb{R}_+$ is $L$-Lipschitz continuous for $\mathbb{R}^K$ equipped with the $\ell_2$-norm.
  There holds
  \begin{align*}
  |\ell(h(\xx), ~ \yy) - \ell(h'(\xx'), ~ \yy)| \leq L\|h(\xx) - h'(\xx')\|_2,
  \end{align*}
  where $(\xx, \yy) \in \mathcal{X} \times \mathcal{Y}, ~ \forall ~\xx' \in \mathcal{X}$ and $h,h': \mathcal{X} \to \mathcal{Y}.$
 \end{assumption}
This assumption is standard in vector-valued learning and can be extended to structured prediction \cite{cortes2016structured}.
Using the Lipschitz condition and contraction lemma for Rademacher complexity of vector-valued learning proven in \cite{cortes2016structured,maurer2016vector}, we further establish the connection between local Rademacher complexity of the loss space and hypothesis space.

\section{Theoretical Analysis}
\label{sec.theory}
In this section,
we study the generalization ability of vector-valued learning.
Firstly, a general excess risk bound is derived using local Rademacher complexity and unlabeled samples.
Then, for the kernel hypotheses, an estimate of local Rademacher complexity is explored based on the eigenvalues decomposition of the normalized kernel matrix.
Thus, an explicit excess risk bound is derived.
Meanwhile, for the linear hypotheses, we bound local Rademacher complexity based on singular values decomposition of the weight matrix $\WW$ and then provide an explicit excess risk bound.
Our analysis is general and applicable to a broad family of vector-valued functions, as long as the loss function is Lipschitz continuous and bounded.

\subsection{General Bound for Local Rademacher Complexity}
\begin{lemma}[Lemma 5 of \cite{cortes2016structured}]
 \label{lem.contraction}
 Under Assumption \ref{ass.lipschitz}, the following contraction inequality exists
 \begin{align*}
 \mathcal{R} (\mathcal{L}_r) \leq \sqrt{2} L \mathcal{R}(\mathcal{H}_r).
 \end{align*}
\end{lemma}

The contraction lemma above has been proven based on Khintchine inequalities in Lemma 5 of \cite{cortes2016structured} and Theorem 3 of \cite{maurer2016vector}.
The contraction inequality in Lemma \ref{lem.contraction} is the key tool for analyzing vector-valued output functions, bridging the gap between Rademacher complexity of the loss space and hypothesis space.
We can then make use of unlabeled data because the richness measure of the hypothesis space is output-independent, always leading to tighter error bounds.

\begin{theorem}[Excess risk bound of vector-valued learning]
 \label{thm.general-lrc-bound}
 Assume the loss function satisfies Assumption \ref{ass.lipschitz}.
 Let $\psi(r)$ be a sub-root function and $r^*$ be the fixed point of $\psi.$
 Fix $\delta \in (0, 1)$ and assume that, for any $r \geq r^*,$
 \begin{align}
 \label{eq.lrc-bound.eq1}
 \psi(r) \geq \sqrt{2} B L \mathcal{R}(\mathcal{H}_r).
 \end{align}
 Then, with a probability of at least $1-\delta,$
 \begin{align}
 \label{eq.lrc-bound.eq2}
 \mathcal{E}(\widehat{h}) - \mathcal{E}(h^*)
 \leq \frac{705}{B}r^* + \frac{49B\log(1/\delta)}{n},
 \end{align}
 where $\widehat{h}, h^*$ are the estimators with the minimal empirical loss and the minimal expected loss, respectively.
\end{theorem}

The above theorem provides a general excess risk bound for semi-supervised vector-valued functions based on local Rademacher complexity.
The classic local Rademacher complexity based bounds \cite{bartlett2005local} estimate the complexity on the loss space $\mathcal{R}(\mathcal{L}_r).$ Note that $\mathcal{R}(\mathcal{L}_r)$ is label-dependent so can only be estimated on labeled samples $\mathcal{D}_l,$ whose convergence rate is $\mathcal{O}(1/\sqrt{n}).$
In contrast, we estimate Rademacher complexity of the hypothesis space $\mathcal{R}(\mathcal{H}_r),$ which is label-independent and so it can be estimated on both labeled samples $\mathcal{D}_l$ and unlabeled samples $\mathcal{D}_u,$ where the convergence rate is $\mathcal{O}(1/\sqrt{n+u})$.
Therefore, using the contraction inequality in Lemma \ref{lem.contraction}, we introduce $\mathcal{R}(\mathcal{H}_r)$ instead of $\mathcal{R}(\mathcal{L}_r)$ to derive tighter error bounds.
\begin{remark}
When there is no labeled data, setting $u=0$ coincides with local Rademacher complexity bounds in supervised settings \cite{li2018multi}.
The convergence rate of these bounds depends on the $\mathcal{R}(\mathcal{H}_r)$ and $\mathcal{O}({1}/{n})$ terms, so it cannot be faster than $\mathcal{O}({1}/{n}).$
The number of unlabeled instances plays a key role in making the error bounds close to $\mathcal{O}({1}/{n})$ during the generalization analysis.
\end{remark}

\subsection{Excess Risk Bound From Kernel Perspective}
\label{subsection.kernel_bound}
In this section, we study the generalization performance of vector-valued functions with kernel hypotheses.
We first present an estimate of local Rademacher complexity on all data in Theorem \ref{thm.estimate-lrc-kernel}, which depends primarily on the tail sum of the eigenvalues of the normalized kernel matrix.
Then, we derive an explicit excess risk bound based on local Rademacher complexity (Corollary \ref{cor.lrc-bound-kernel}) for vector-valued functions with a faster convergence rate, by applying Theorem \ref{thm.estimate-lrc-kernel} to Theorem \ref{thm.general-lrc-bound}.

\begin{theorem}[Local Rademacher complexity for kernel estimators]
 \label{thm.estimate-lrc-kernel}
 Let $\mathcal{H}_r$ be the local hypotheses space defined in \eqref{eq.localized-hypothesis-space} and $\|\WW\|_* \leq 1$.
 Let eigenvalue decomposition be $\kappa(\xx, \xx') = \sum_{j=1}^\infty \lambda \psi_j(\xx)^T \psi_j(\xx')$, where its eigenvalues be $(\lambda_j)_{j=1}^{\infty}$ in a nonincreasing order.
 For any $r > 0,$ there holds
 \begin{align*}
 \mathcal{R}(\mathcal{H}_r)
 \leq 2\sqrt{\frac{1}{n + u}\min_{\theta \geq 0}\Big(\frac{\theta r}{4L^2} + \sum_{j > \theta} \lambda_j \Big)}.
 \end{align*}
\end{theorem}

Theorem \ref{thm.estimate-lrc-kernel} demonstrates that local Rademacher complexity is determined by the tail sum of eigenvalues, where the eigenvalues are truncated at the "cut-off point" $\theta.$
\begin{remark}
  Notably, local Rademacher complexity is independent from the number of classes $K$, because the constraints on $\WW \in \mathcal{S} \times \mathbb{R}^K$ (e.g. $\WW_* \leq 1$) are actually related to the dimensionality of the output space $K$. 
  When $K$ is bigger, the constraints are relatively stricter.
\end{remark}

\begin{corollary}[Excess risk bound for kernel estimators]
 \label{cor.lrc-bound-kernel}
 Assume the loss function satisfies Assumption \ref{ass.lipschitz} and kernel estimators satisfy $\sup_{\xx \in \mathcal{X}} \kappa(\xx,\xx) \leq 1$ and $\|\WW\|_* \leq 1$.
 With a probability of at least $1 - \delta,$ it holds that
 \begin{align}
  \label{eq.kernel_bound}
 \mathbb{E} [\mathcal{E}(\widehat{h}) - \mathcal{E}(h^*)]
 \leq~ c_{L,B} \left(r^* + \frac{\log (1/\delta)}{n}\right),
 \end{align}
 where, for the fixed point, it holds that
 \begin{align*}
 r^* \leq \min_{\theta \geq 0} \left(\frac{\theta}{n+u} + \sqrt{\frac{1}{n+u} \sum_{j > \theta} \lambda_j}\right),
 \end{align*}
 where $c_{L,B}$ a constant only depending on $L$ and $B$.
\end{corollary}

The convergence rate of the above bound depends on the quantity $\sqrt{\frac{1}{n+u} \sum_{j > \theta} \lambda_j}$, which can be estiamted as follows:

\textbf{1) The worst case (when $\theta = 0$).} The complexity degrades into the \textit{global} Rademacher complexity, depending on the trace of the kernel $\kappa.$
Then, the convergence rate of excess risk bound is $\mathcal{E}(\widehat{h}) - \mathcal{E}(h^*) = \mathcal{O}\left(\sqrt{\frac{1}{n + u}} + \frac{1}{n}\right)$.

\textbf{2) Finite-rank kernel.}
When the kernel $\kappa$ has a finite rank $\theta$ such that its eigenvalues satisfy $\lambda_j = 0$ for all $j > \theta$, the tail sum of eigenvalues is zero.
Indeed, a lot of common kernels are finite-rank kernels, e.g. the linear kernel and polynomial kernel.
The linear kernel $\kappa(\xx, \xx') = \langle\xx, \xx'\rangle$ has a rank of at most $\theta=d$.
For a polynomial kernel $\kappa(\xx, \xx') = (\langle \xx, \xx'\rangle + 1)^p$ with degree $p$, its rank is at most $\theta=p+1$.
Thus, the rate of the fixed point is inversely proportional to the number of samples, that is $r^* = \mathcal{O}\left(\frac{\theta}{n+u}\right).$

\textbf{3) Exponentially decaying eigenvalues.} 
When the eigenvalues of the normalized kernel matrix $\mathbf{K}$ decay exponentially $\sum_{j>\theta} \lambda_j = \mathcal{O}(\exp(-\theta)),$
such as for Gaussian kernels \cite{bartlett2005local}, then by truncating a thresholding with $\theta = \log (n+u)$
it holds that $r^* = \mathcal{O}\left(\frac{\log (n+u)}{n+u}\right).$

Both finite-rank kernels and kernels with exponentially decaying eigenvalues have an $r^*$ that mainly depends on $\mathcal{O}(1/(n+u)),$ which is much smaller than $\mathcal{O}(1/n).$
In these cases, the excess risk bound (\ref{eq.kernel_bound}) provides a linear dependence on the labeled sample size $\mathcal{E}(\widehat{h}) - \mathcal{E}(h^*) = \mathcal{O}\Big(\frac{1}{n}\Big)$,
yielding much stronger generalization bounds.
A similar analytical procedure is also used in classic local Rademacher complexity literature \cite{bartlett2005local,xu2016local}.

\subsection{Excess Risk Bound From Linear Perspective}
In this section, we study the local Rademacher complexity bound for $h(\xx) = \WW^T\phi(\xx),$ using the singular values decomposition (SVD) of the weight matrix $\WW.$
The result (Theorem \ref{thm.estimate-lrc-linear}) shows that local Rademacher complexity can be bounded by the tail sum of the singular values of $\WW.$
Combining Theorem \ref{thm.general-lrc-bound} and Theorem \ref{thm.estimate-lrc-linear}, we obtain a tighter generalization error bound (Corollary \ref{cor.lrc-bound-linear}).
\begin{theorem}[Local Rademacher complexity for linear estimators]
 \label{thm.estimate-lrc-linear}
 Let the SVD decomposition be $\WW={\boldsymbol U}{\boldsymbol \Sigma}{\boldsymbol V}^T.$
 ${\boldsymbol U}$ and ${\boldsymbol V}$ are unitary matrices,
 and ${\boldsymbol \Sigma}$ is diagonal with singular values $\{\tilde{\lambda}_j\}$ in descending order.
 Under Assumption \ref{ass.lipschitz}, assuming $\mathbb{E} ~ [\xx^T\xx] \leq 1$ and $\|\WW\|_* \leq 1$, the local Rademacher complexity $\mathcal{R}(\mathcal{H}_r)$ for linear hypotheses is upper bounded by
 \begin{align*}
 \mathcal{R}(\mathcal{H}_r) \leq
 2\sqrt{\frac{1}{n + u}\min_{\theta \geq 0}\Big(\frac{\theta r}{4L^2} + \sum_{j > \theta} \tilde{\lambda}_j^2\Big)}.
 \end{align*}
\end{theorem}

The above theorem estimates local Rademacher complexity for linear estimators.
We find that the first term of the right side of the inequality $\theta r/ (4L^2)$ is a constant, such that
local Rademacher complexity is determined by the tail sum of squared singular values of the weight matrix $\WW.$

\begin{remark}
For the kernel hypotheses, local Rademacher complexity can be bounded by the tail sum of the eigenvalues of the normalized kernel matrix $\mathbf{K}$ \cite{bartlett2005local,cortes2013learning,li2018multi}.
Similarly, for the linear hypotheses, Theorem \ref{thm.estimate-lrc-linear} shows that local Rademacher complexity can be bounded by the singular values of the weight matrix $\WW.$
\end{remark}

\begin{corollary}[Excess risk bound for linear estimators]
 \label{cor.lrc-bound-linear}
 Assume that the loss function satisfies Assumption \ref{ass.lipschitz}.
 Let $\mathbb{E} ~ [\xx^T\xx] \leq 1$ and the trace norm $\|\WW\|_* \leq 1$.
 With a probability of at least $1 - \delta,$ it holds that
 \begin{align}
  \label{eq.linear_bound}
 \mathcal{E}(\widehat{h}) - \mathcal{E}(h^*)
 \leq~ \tilde{c}_{L,B} \left(\tilde{r}^* + \frac{\log (1/\delta)}{n}\right),
 \end{align}
 where, for the fixed point, it holds that
 \begin{align*}
 \tilde{r}^* \leq \min_{\theta \geq 0} \left(\frac{\theta}{n+u} + \sqrt{\frac{1}{n+u} \sum_{j > \theta} \tilde{\lambda}_j^2}\right),
 \end{align*}
 where $(\tilde{\lambda}_j)_{j=1}^\infty$ are the singular values of $\WW$ and $\tilde{c}_{L,B}$ is a constant only depending on $L$ and $B.$
\end{corollary}

The convergence rate of Corollary \ref{cor.lrc-bound-linear} depends on the quantity $\sqrt{\frac{1}{n+u} \sum_{j > \theta} \tilde{\lambda}_j^2}$ and we estimate it as follows:

\textbf{1) The worst case ($\theta = 0$).} The fixed point $\tilde{r}^*$ becomes relevant to \textit{global} Rademacher complexity, at $\mathcal{O}(1/\sqrt{n + u}).$
Thus, the convergence rate is $\mathcal{E}(\widehat{h}) - \mathcal{E}(h^*) = \mathcal{O}\Big(\frac{1}{\sqrt{n+u}} + \frac{1}{n}\Big).$

\textbf{2) Faster convergence.}
Similar to the analysis in Subsection \ref{subsection.kernel_bound}, when $\WW$ has a finite rank or its eigenvalues decay exponentially, the fixed point $\tilde{r}^*$ mainly depends on $\tilde{r}^* \leq \theta/(n+u)$ in (\ref{eq.linear_bound}).
We obtain better results with a fast convergence rate $\mathcal{E}(\widehat{h}) - \mathcal{E}(h^*) = \mathcal{O}\Big(\frac{1}{n}\Big).$
Similar theoretical results for \textit{linear} estimators were presented for multi-label \cite{xu2016local} and multi-class in our previous work \cite{li2019multi}.

\begin{remark}
 \label{re.appropriate-theta}
 The tail sum of eigenvalues or singular values are often used to bound local Rademacher complexity \cite{bartlett2005local,cortes2013learning,xu2016local}.
 As discussed in \cite{cortes2013learning}, the choice of threshold $\theta$ is very important.
 If $\theta$ is too small, the \textit{local} Rademacher complexity will be close to the \textit{global} Rademacher complexity.
 If $\theta$ is too big, the \textit{local} Rademacher complexity will be nearly constant.
 For the finite-rank matrix, we simply set $\theta$ equal to rank, so the tail sum is zero.
 For other cases, the optimal $\theta$ is obtained by making two terms in Theorem \ref{thm.estimate-lrc-kernel} and Theorem \ref{thm.estimate-lrc-linear} equal, typically
 $\frac{\theta r}{4L^2} = \sum_{j>\theta}\lambda_j$ for the kernel hypotheses
 and
 $\frac{\theta r}{4L^2} = \sum_{j>\theta}\tilde{\lambda}_j^2$ for the linear hypotheses.
\end{remark}

\begin{remark}
  Theorem \ref{thm.estimate-lrc-linear} provides an excess risk bound for general vector-valued learning, as long as the estimator can be formed as $h(\xx) = \WW^T\phi(\xx)$.
  Therefore, Theorem \ref{thm.estimate-lrc-linear} holds for lots of approaches, including deep neural networks, random features, decision trees, and so on.
\end{remark}

\section{Comparisons with Related Work}
\label{sec.theory_comparison}
In this section, we first introduce the typical data-dependent error bounds of general vector-valued functions and compare them with our bounds. Then, we present traditional works for two special cases, multi-class classification and multi-label learning, and then compare their statistical properties with ours.
Specifically, from the following comparisons, our theoretical results estimate \textit{excess risk} $\mathcal{E}(\widehat{h}) - \mathcal{E}(h^*)$ while others estimate \textit{generalization error} $\mathcal{E}(h) - \widehat{\mathcal{E}}(h).$

\subsection{General Vector-valued Functions}

\begin{table}[t]
  \caption{
    Data-dependent error bounds for vector-valued functions (VV).
    $\dagger$ indicates with unlabeled data and $\ddagger$ represents \textit{excess} risk bounds.
    }
  \label{tab.comparison-vv-bounds}
  \centering
  \small
  \setlength\extrarowheight{5pt}
  \begin{tabular}{l|l|c|c}
      \hline
      Bounds          & Condition            & Worst Case                                                                & Special Case                     \\ \hline
      GRC for VV \cite{cortes2016structured}         & Lipschitz continuity           & \multicolumn{2}{c}{Kernel:  $\mathcal{O}\big(\sqrt{\frac{\log K}{n}}\big)$ ~ Linear: $\mathcal{O}\big(\sqrt{\frac{K}{n}}\big)$} \\ \hline
      GRC for VV \cite{maurer2016vector}          & Lipschitz continuity           & \multicolumn{2}{c}{Kernel:  $\mathcal{O}\big(\frac{1}{\sqrt{n}}\big)$ ~ Linear: $\mathcal{O}\big(\sqrt{\frac{K}{n}}\big)$} \\ \hline
      LRC for VV \cite{wu2021fine}        & Strong convexity             &\multicolumn{2}{c}{$\mathcal{O}\big(\frac{\log^3(nK)}{n}\big)$} \\ \hline
      LRC for VV \cite{li2021towards}      & Smoothness               &\multicolumn{2}{c}{$\mathcal{O}\big(\frac{\log^3(nK)}{n}\big)$} \\ \hline
      LRC for Kernel VV (Corollary \ref{cor.lrc-bound-kernel}) $\dagger\ddagger$  & Lipschitz continuity   & $\mathcal{O}\big(\frac{1}{\sqrt{n+u}} + \frac{1}{n}\big)$                          & $\mathcal{O}\big(\frac{1}{n}\big)$ \\\hline
      LRC for Linear VV (Corollary \ref{cor.lrc-bound-linear}) $\dagger\ddagger$  & Lipschitz continuity & $\mathcal{O}\big(\frac{1}{\sqrt{n+u}} + \frac{1}{n}\big)$                          & $\mathcal{O}\big(\frac{1}{n}\big)$ \\ \hline
  \end{tabular}
\end{table}

In this paper, the contraction inequality in Lemma \ref{lem.contraction} is a key step in our analysis to connect Rademacher complexity of loss function classes and Rademacher complexity of the hypothesis space.
Table \ref{tab.comparison-vv-bounds} shows comparisons of data-dependent error bounds.
For kernelized vector-valued functions, the convergence rate of error bounds in \cite{cortes2016structured} and \cite{maurer2016vector} are $\mathcal{O}\big(\sqrt{{\log K}/{n}}\big)$ and $\mathcal{O}\big(\sqrt{1/{n}}\big)$ respectively. 
More recently, there are statistical advances for vector-valued learning or more general structured prediction.
For example, \cite{ciliberto2020general} provided the excess risk bound for structured prediction whose convergence rate is $\mathcal{O}(n^{-1/3})$ in the general cases, while \cite{brogat2022vector} derived $\mathcal{O}(n^{-1/4})$ excess risk bound for the vector-valued regression.
Compared to the above literature, we improve the kernel bounds to $\mathcal{O}\big({1}/{\sqrt{n+u}} + {1}/{n}\big)$ in this paper.
Even though $\mathcal{O}(1/n)$ excess risk bounds are achieved recently for vector-valued learning \cite{wu2021fine} and structured prediction \cite{li2021towards}, they require some strict conditions on the loss function, for example, strong convexity \cite{wu2021fine} and smoothness condition \cite{li2021towards}.

For linear vector-valued functions, the learning error rates of \cite{cortes2016structured} and \cite{maurer2016vector} are both $\mathcal{O}\big(\sqrt{{K}/{n}}\big),$
while the theoretical analysis in Corollary \ref{cor.lrc-bound-linear} provides a much sharper learning rate, even in the worst case, of $\mathcal{O}\big({1}/{\sqrt{n+u}} + {1}/{n}\big).$
What's more, in the benign cases, the convergence rates of vector-valued learning in both the kernel space and linear space are $\mathcal{O}(1/n)$,
which is much faster than the rate of error bounds in \cite{cortes2016structured,maurer2016vector}.
Meanwhile, our bounds are independent from the vector size $K,$ thus they are more suitable when $K$ is large.
To make the hypothesis space smaller, we explore the \textit{local} Rademacher complexity instead of the \textit{global} one. 
Meanwhile, to reduce the output-independent complexity $\mathcal{R}(\mathcal{H}_r),$ we make use of unlabeled samples.
Based on these two aspects, we obtain significant statistical gains.

\begin{table}[t]
    \caption{
      Data-dependent error bounds for multi-class classification (MC).
        $\dagger$ indicates with unlabeled data and $\ddagger$ represents \textit{excess} risk bounds.
        }
    \label{tab.comparison-mc-bounds}
    \centering
    \small
    \setlength\extrarowheight{5pt}
    \begin{tabular}{l|c|c}
        \hline
        Bounds                                              & Worst Case                                                                & Special Case                     \\ \hline
        GRC for Kernel MC \cite{cortes2013multi}                    & \multicolumn{2}{c}{$\mathcal{O}\big(\frac{ K}{\sqrt{n}}\big)$}                                                        \\ \hline
        GRC for Kernel MC \cite{lei2019data}                    & \multicolumn{2}{c}{$\mathcal{O}\big(\frac{\log K}{\sqrt{n}}\big)$}                                                        \\ \hline
        GRC for Kernel MC \cite{maximov2018rademacher} $\dagger$                   & \multicolumn{2}{c}{$\mathcal{O}\big(\sqrt{\frac{K}{n}}+K\sqrt{\frac{K}{u}}\big)$}                                    \\ \hline
        LRC for Kernel MC \cite{li2018multi} & \multicolumn{2}{c}{$\mathcal{O}\big(\frac{\log^2 K}{n}\big)$} \\ \hline
        LRC for Linear MC \cite{li2019multi} $\dagger$     & $\mathcal{O}\big(\frac{K}{\sqrt{n+u}} + \frac{1}{n}\big)$                          & $\mathcal{O}\big(\frac{1}{n}\big)$     \\ \hline
        LRC for Kernel VV 
        (Corollary \ref{cor.lrc-bound-kernel}) $\dagger$ $\ddagger$     & $\mathcal{O}\big(\frac{1}{\sqrt{n+u}} + \frac{1}{n}\big)$                          & $\mathcal{O}\big(\frac{1}{n}\big)$ \\\hline
        LRC for Linear VV 
        (Corollary \ref{cor.lrc-bound-linear}) $\dagger$ $\ddagger$    & $\mathcal{O}\big(\frac{1}{\sqrt{n+u}} + \frac{1}{n}\big)$                          & $\mathcal{O}\big(\frac{1}{n}\big)$ \\ \hline
    \end{tabular}
  \end{table}
\subsection{Multi-class Classification}
Based on data-dependent richness measures, the generalization ability of multi-class classification has been well-studied \cite{lei2019data,maximov2018rademacher}.
As illustrated in Table \ref{tab.comparison-mc-bounds}, our excess bounds are among the sharpest results both in the kernel space and linear space.
Generalization error bounds using Rademacher complexity for multi-class classification
were explored in \cite{koltchinskii2001some,cortes2013multi} and the convergence rate of these error bounds is $\mathcal{O}\big({K}/{\sqrt{n}}\big).$
Using Gaussian complexity (GC) and Slepian's Lemma, a generalization error bound with logarithmic dependence on $K$ was derived in \cite{lei2019data}, whose convergence rate is $\mathcal{O}\big({\log K}/{\sqrt{n}}\big)$.
Making use of unlabeled instances,
Maximov et al. presented a data-dependent error bound for semi-supervised multi-class classification with the rate $\mathcal{O}(\sqrt{K/n}+K\sqrt{K/u})$ \cite{maximov2018rademacher}.

Although \textit{global} Rademacher complexity is widely used in generalization analysis,
it does not take into consideration the fact that the hypothesis selected by a learning algorithm typically
belongs to a small favorable subset of all hypotheses \cite{bartlett2005local,cortes2013learning}.
In contrast, local Rademacher complexity evaluates richness on a small subset of the hypothesis space,
which is often used to obtain better error bounds for binary classification and regression.
Our previous work \cite{li2018multi} introduced local Rademacher complexity into the multi-class classification
and obtained a reciprocal dependence on the number of labeled samples $n$ for the first time.
However, this paper is quite different from our previous work \cite{li2018multi} in both its conditions and technical details referred in Section \ref{sec.related_work}.

\subsection{Multi-Label Learning}
The \textit{global} Rademacher complexity was introduced to the generalization analysis of multi-label learning in \cite{yu2014large}, obtaining generalization error bounds of $\mathcal{O}(1/\sqrt{n}).$
Generally, the \textit{global} Rademacher complexity is bounded by the trace norm of $\WW.$
Further, using the \textit{local} Rademacher complexity, \cite{xu2016local} improved the error bounds.
Local Rademacher complexity of multi-label learning can be determined by the tail sum of singular values of $\WW$, where a fast convergence rate $\mathcal{O}(1/n)$ is obtained when the rank of $\WW$ is finite or its singular values decay exponentially.
Both \cite{yu2014large} and \cite{xu2016local} explored the generalization ability of multi-label learning in the linear space, while our theoretical results include both the linear and nonlinear estimators.
Table \ref{tab.comparison-ml-bounds} compares data-dependent generalization bounds for multi-label learning, showing that our results are much better than former works due to the use of local Rademacher complexity and unlabeled data.

\begin{remark}
  Previous literature provided the generalization bounds for semi-supervised learning in multiview scenarios \cite{sun2010sparse,sun2011multi}.
  Furthermore, one can derive sharper generalization bounds for semi-supervised multiview learning using local Rademacher complexity.
  Specifically, by restricting the capacity of loss space to guarantee a small ball around the optimal hypothesis, one can make use of the self-bounding lemma and the property of the sub-root function \cite{bartlett2005local}  that can guarantee $\mathcal{O}(1/n)$ bounds.
\end{remark}

\begin{table}[t]
  \caption{
    Data-dependent error bounds for multi-label learning (ML).
    $\dagger$ indicates with unlabeled data and $\ddagger$ represents \textit{excess} risk bounds.
    }
  \label{tab.comparison-ml-bounds}
  \centering
  \small
  \setlength\extrarowheight{5pt}
  \begin{tabular}{l|c|c}
    \hline
    Bounds                                              & Worst Case                                                            & Special Case                     \\ \hline
    GRC for Linear ML \cite{yu2014large}                    & \multicolumn{2}{c}{$\mathcal{O}\big(\frac{1}{\sqrt{n}}\big)$}                                    \\ \hline
    LRC for Linear ML \cite{xu2016local} & $\mathcal{O}\big(\frac{1}{\sqrt{n}}\big)$                          & $\mathcal{O}\big(\frac{1}{n}\big)$     \\ \hline
    LRC for Kernel VV 
    (Corollary \ref{cor.lrc-bound-kernel}) $\dagger$ $\ddagger$     & $\mathcal{O}\big(\frac{1}{\sqrt{n+u}} + \frac{1}{n}\big)$                          & $\mathcal{O}\big(\frac{1}{n}\big)$     \\ \hline
    LRC for Linear VV 
    (Corollary \ref{cor.lrc-bound-linear}) $\dagger$ $\ddagger$    & $\mathcal{O}\big(\frac{1}{\sqrt{n+u}} + \frac{1}{n}\big)$                          & $\mathcal{O}\big(\frac{1}{n}\big)$     \\
    \hline
  \end{tabular}
\end{table}

\begin{algorithm*}[t]
  \caption{\small Local Rademacher based Semi-supervised Vector-valued Learning (\texttt{LSVV})}
  \label{alg.lsvt}
  \begin{algorithmic}[1]
      \REQUIRE Labeled dataset $\mathcal{D}_l$ and unlabeled dataset $\mathcal{D}_u.$ Initialized matrix $\WW_1=\boldsymbol{0}$.
      Stop iteration number $T.$
      Feature mapping $\phi: \mathbb{R}^d \to \mathbb{R}^D.$
      Parameters: $\theta, \tau_A, \tau_I, \tau_S, \eta_t.$
      \ENSURE $\WW_{T+1}$
      \STATE Perform feature mappings on all samples: $\widetilde{\mathbf{X}} = \phi(\mathbf{X}) \in \mathbb{R}^{D \times (n+u)}$ with $\mathbf{X} = (\xx_i)_{i=1}^{n+u} \in \{\mathcal{D}_l \bigcup \mathcal{D}_u\}$.
      \STATE Compute Laplacian matrix $\boldsymbol{L} \in \mathbb{R}^{(n+u) \times (n+u)}$ for all mapped examples $\widetilde{\mathbf{X}}$.
      \STATE Compute the term $\boldsymbol{G} = \widetilde{\mathbf{X}}\boldsymbol{L}\widetilde{\mathbf{X}}^T \in \mathbb{R}^{D \times D}$.
      \FOR{$t = 1, 2, \cdots, T$}
      \STATE Select a batch of labeled examples $(\xx_i, ~ \yy_i)_{i=1}^m \in \mathcal{D}_l.$
      \STATE Compute the gradient $\nabla g(\WW_t)$ on the batch
      \begin{equation}
        \label{alg.eq-gradient}
        \nabla g(\WW_t)=\frac{1}{m}\sum\limits_{i=1}^{m} \frac{\partial ~\ell(h(\xx_i), ~ \yy_i)}{\partial~ \WW_t}
        + 2\tau_A \WW_t
        + 2\tau_I \boldsymbol{G} \WW_t.
      \end{equation}
      \STATE Update the weight only with the differentiable part $g(\WW)$
      \begin{align}
        \label{alg.eq-upate-first}
        \boldsymbol{Q}_t = \WW_t-\eta_t \nabla g(\WW_t).
      \end{align}
      \STATE Compute the SVD decomposition
      ${\boldsymbol U}{\boldsymbol \Sigma}{\boldsymbol V}^T = \boldsymbol{Q}_t.$
      \STATE Update $\WW_{t+1}$ by reducing first $\theta$ singular values
      \begin{align}
        \label{alg.eq-upate-second}
        \WW_{t+1}={\boldsymbol U}{\boldsymbol \Sigma}_{\tau}^\theta{\boldsymbol V}^T ~~\text{where}~~ \tau = \eta_t \tau_S
      \end{align}
      \ENDFOR
  \end{algorithmic}
\end{algorithm*}  

\section{Algorithm}
\label{sec.algorithm}
Based on our theoretical analysis, we present a unified learning framework to minimize the empirical loss, local Rademacher complexity, and manifold regularization at the same time.
Local Rademacher complexity is bounded by the tail sum of eigenvalues for kernel methods and singular values for linear models.
Manifold regularization is employed to make use of unlabeled instances.
Then, to solve the minimization objective, with adaptive learning rates, we use proximal gradient descent optimization methods and update partial singular values according to thresholding.

\subsection{Learning Framework}
We modify the structural risk minimization (SRM) learning framework with two additional terms: a manifold regularizer to make use of unlabeled samples and a regularizer term to bound local Rademacher complexity.

\subsubsection{Manifold Regularization}
Consider a similarity matrix $\SS$ on all $n+u$ samples, where $\SS_{ij}$ represents the similarity between $\xx_i$ and $\xx_j,$ defined by the binary weights for k-nearest neighbors or the heat kernel $\SS_{ij}=\exp(-\|\xx_i - \xx_j\|^2/\sigma^2).$
To make use of unlabeled data, we define the cost function (manifold regularization) as
\begin{equation}
 \label{eq.alg-laplacian}
 E(h) = \sum_{i,j=1}^{n+u} \SS_{ij}\|h(\xx_i)-h(\xx_j)\|_2^2
 = \text{trace}(\WW^T\widetilde{\mathbf{X}}\boldsymbol{L}\widetilde{\mathbf{X}}^T\WW),
\end{equation}
where $\widetilde{\mathbf{X}} \in \mathbb{R}^{D\times(n+u)}$ corresponds to a feature mapping $\phi$ on all samples, graph Laplacian $\boldsymbol{L}=\boldsymbol{D}-\SS$
and $\boldsymbol{D}$ is a diagonal matrix with $\boldsymbol{D}_{ii}=\sum_{j=1}^{n+u} \SS_{ij}.$

\subsubsection{Local Rademacher Complexity Term}
Motivated by theoretical results (Corollary \ref{cor.lrc-bound-kernel} and Corollary \ref{cor.lrc-bound-linear}), we use the tail sum of the eigenvalues of the normalized kernel matrix $\mathbf{K}$ or the tail sum of the squared singular values of weight matrix $\WW$ to bound local Rademacher complexity.
Note that the minimization of the tail sum of singular values $\sum_{j > \theta} \tilde{\lambda}_j(\WW)$ is equivalent to the tail sum of squared singular values $\sum_{j > \theta} \tilde{\lambda}_j^2(\WW)$.
For the sake of simplicity, we use singular values form.
The regularizer used to minimize local Rademacher complexity is
\begin{equation}
 \label{eq.alg-tail-sum}
 \begin{aligned}
 T(h) =
 \begin{cases}
 \sum_{j > \theta} \lambda_j(\mathbf{K}), ~~ &\text{for kernel hypotheses,}\\
 \sum_{j > \theta} \tilde{\lambda}_j(\WW), ~~ &\text{for linear hypotheses,}
 \end{cases}
 \end{aligned}
\end{equation}
where $\lambda_j(\mathbf{K})$ represents the $j$-th largest eigenvalue of the kernel matrix $\mathbf{K}$ and $\tilde{\lambda}_j(\WW)$ represents the $j$-th largest singular value of $\WW.$

\subsubsection{Minimization Objective}
Then, combining the ERM learning framework with the Laplacian regularization (\ref{eq.alg-laplacian}) and local Rademacher complexity term (\ref{eq.alg-tail-sum}), we define the minimization objective as
\begin{equation}
 \label{eq.alg-obj}
 \mathop{\arg\min}\limits_{h\in\mathcal{H}_r}
 \underbrace{\frac{1}{n} \sum_{i=1}^{n} \ell(h(\xx_i), ~ \yy_i)
 + \tau_A \|\WW\|^2_F
 + \tau_I E(h)}_{g(\WW)}
 +\tau_S T(h),
\end{equation}
where $\tau_A, \tau_I$ and $\tau_S$ are regularization parameters, $E(h)$ is the Laplacian regularization and $T(h)$ is the regularizer on local Rademacher complexity.

For kernel hypotheses, the tail sum of the eigenvalues of the kernel is commonly used to estimate local Rademacher complexity.
However, the tail sum of eigenvalues for one single kernel is a constant, so it does not influence the learning model if we add this term to the objective.
However, local Rademacher complexity of multiple kernel learning (MKL) is undetermined, and thus our previous work \cite{li2018multi} introduced local Rademacher complexity to improve the performance of multi-class MKL.
Yet, the optimization of multi-class MKL was overly complicated and inefficient.
In this paper, we adopt an efficient feature mapping-based linear estimator to approximate single kernel methods \cite{rahimi2007random} rather than using ineffective MKL.
We then define local Rademacher complexity term in a general form in \eqref{eq.alg-tail-sum}, $T(h) = \sum_{j > \theta} \tilde{\lambda}_j(\WW),$ for both linear estimators and approximate kernel estimators.

\begin{remark}[Computational costs]
  Semi-supervised learning and local Rademacher complexity offer statistical benefits for vector-valued learning in Theorem \ref{thm.estimate-lrc-linear}, but also bring computational burdens in Algorithm \ref{alg.lsvt}.
  Here, we analyze the computational complexity of Algorithm \ref{alg.lsvt}.
  Before the training, feature mapping costs $\mathcal{O}((n+u)D)$ time and Laplacian matrix requires $\mathcal{O}((n+u) \log k)$ with $k$-nearest neighbors search.
  The computation of $\boldsymbol{G}$ is $\mathcal{O}((n+u)Dk)$, due to Laplacian matrix $\boldsymbol{L}$ is sparse and only have $k$ nonzero elements in each column.
  In each iteration, the computation of the gradient \eqref{alg.eq-gradient} requires $\mathcal{O}(D^2 K)$ due to the term $\boldsymbol{G}\WW_t$,
  the update of the weight \eqref{alg.eq-upate-first} costs $\mathcal{O}(DK)$, while the update of part singular values \eqref{alg.eq-upate-second} requires $\mathcal{O}(D^2K)$ for computing SVD of $\boldsymbol{Q}_t$.
  Therefore, the computational complexity is
  \begin{align*}
    \mathcal{O}\left((n+u)Dk + D^2KT\right),
  \end{align*}
  where the number of iterations is $T$.
  The cost $\mathcal{O}((n+u)Dk)$ is used in data pre-processing, while the cost is $\mathcal{O}(D^2KT)$ for the training, due to both Laplacian regularization and local Rademacher complexity.
  The training time is much smaller than previous MKL-MC \cite{li2018multi}
  and same to semi-supervised neural networks \cite{mojoo2017deep}.
\end{remark}

\begin{table}[t]
  \caption{Statistics of the experimental datasets.}
  \label{tab.datasets}
  \centering
  \small
  \begin{tabular}{l|lrrrrr}
      \toprule
      Task &Datasets   & \# training   & \# testing & \# $d$ & \# $K$       \\ \hline
      \multirow{20}*{\textbf{MC}}
      &iris &105  &45 &5  &3 \\
      &wine &125  &53 &14 &3 \\
      &glass  &150  &64 &10 &6 \\
      &svmguide2  &274  &117  &21 &3 \\
      &vowel  &370  &158  &11 &11 \\
      &vehicle  &593  &253  &19 &4 \\
      &segment  &1617 &693  &19 &7 \\
      &satimage &3105 &1330 &37 &6 \\
      &pendigits  &5246 &2248 &17 &10 \\
      &letter &10500  &4500 &17 &26 \\
      &poker  &17507  &7503 &11 &10 \\
      &shuttle  &27631  &11841  &10 &7 \\
      &Sensorless &40957  &17552  &49 &11 \\
      &MNIST  &42000  &18000  &718  &10 \\
      &connect-4  &47290  &20267  &127  &3 \\
      &acoustic &55177  &23646  &51 &3 \\
      &covtype  &406709 &174303 &55 &7 \\
      \hline
      \multirow{4}*{\textbf{MLC}}
      &scene  &1685 &722  &295  &6 \\
      &yeast  &1692 &725  &104  &14 \\
      &corel5k  &3150 &1350 &500  &374 \\
      &bibtex &5177 &2218 &1837 &159 \\
      \hline
      \multirow{2}*{\textbf{MLR}}
      &rf2  &5376 &2303 &577  &8 \\
      &scm1d  &6863 &2940 &281  &16 \\
      \bottomrule
  \end{tabular}
\end{table}

\subsection{Nonlinear Feature Mapping}
\label{section.rf}
Algorithm \ref{alg.lsvt} 
Nonlinear feature mapping based vector-valued learning covers many popular methods, including (deep) neural networks, random features, the mixture of experts, decision trees, and so on.
Here, we take neural networks and random features as examples.

Assume the neural networks contain $L$ layers, and thus the linear estimator $f: \mathbb{R}^D \to \mathbb{R}^K$ is in the $L$-th layer and the nonlinear feature mapping $\phi: \mathbb{R}^d \to \mathbb{R}^D$ is before the last layer.
For $l \in [L-1]$, the $l$-th layer is also a nonlinear feature mapping $\phi_l: \mathbb{R}^{D_{l-1}} \to \mathbb{R}^{D_l}$ that
\begin{align*}
  \phi_l(\xx) = \varphi_l\Big(\Omega_l^T\phi_{l-1}(\xx) + \boldsymbol{b}_l\Big),
\end{align*}
where $\varphi_l$ is the nonlinear function in the $l$-th layer, $\Omega_l \in \mathbb{R}^{D_{l-1} \times D_l}$ is the weight matrix and $\boldsymbol{b}_l \in \mathbb{R}^{D_l}$ is the bias vector.
Thus, the entire feature mapping is 
$$
  \phi(\xx) = \varphi_{L-1}(\varphi_{L-2}( \cdots \varphi_1(\Omega_1^T\xx + {\boldsymbol b}_1) \cdots) + {\boldsymbol b}_{L-2}).
$$
where the feature space is $\mathcal{S} = \mathbb{R}^{D_{L-1}}$.

The kernel function $\kappa$ admit an integral representation
\begin{align}
    \label{equation.stationary}
    \kappa(\xx, \xx') & = \int_{\Omega} \psi(\xx, \oo)\psi(\xx', \oo) d \pi(\oo),
\end{align}
where $(\Omega, \pi)$ is a probability space and $\psi: \mathcal{X} \times \Omega \to \mathbb{R}$.
To accelerate the solve of kernel methods, we use random features to approximate the kernel $\kappa(\xx, \xx') \approx \langle \phi(\xx), \phi(\xx') \rangle$.
Using Mento Carlo sampling, we then define an explicit random features $\phi: \mathbb{R}^d \to \mathbb{R}^D$ as
\begin{align}
 \label{eq.random-features}
 \phi(\xx) = \sqrt{\frac{1}{D}} \Big[\psi(\xx, \oo_1), \cdots, \psi(\xx, \oo_D)\Big]^T,
\end{align}
where $\oo_1, \cdots, \oo_D \in \Omega$ are drawn i.i.d from the probability density $\pi(\Omega)$.
Random features in \eqref{eq.random-features} are used to approximate any kinds of kernel funcitons \eqref{equation.stationary}.
Random Fourier features are used to approximate shift-invariant kernels \cite{rahimi2007random}.

\begin{remark}
  Kernel methods are equivalent to wide neural networks with one single hidden layer \cite{neal2012bayesian}.
  Similar, random features can be seen as finite wide neural networks but only one hidden layer.
  Theorem \ref{thm.estimate-lrc-linear} and Algorithm \ref{alg.lsvt} holds for general vector-valued learning, as long as the learners can be formed as $h(\xx) = \WW^T\phi(\xx)$.
\end{remark}

\section{Experiments}
\label{sec.experiment}
We set up four experiments to evaluate the empirical behavior of the proposed algorithm \texttt{LSVV} :
(1) average test error of multi-class classification;
(2) empirical performance of multi-label learning (test error for multi-label classification and RMSE for multi-label regression);
(3) influence of the thresholding $\theta;$
(4) influence of the labeled rate.

\subsection{Experimental Setup}

\textbf{Datasets.}
As demonstrated in Table \ref{tab.datasets}, we use a variety of public benchmark datasets, in which the number of points ranges from hundreds to hundreds of thousands. These datasets cover two kinds of applications:
(1) Multi-class classification (\textbf{MC}) with 17 datasets \footnote{\url{https://www.csie.ntu.edu.tw/~cjlin/libsvmtools/datasets/}},
(2) Multi-label learning, including four datasets for multi-label classification (\textbf{MLC}) \footnote{\url{https://mulan.sourceforge.net/datasets-mlc.html}} and two datasets for multi-label regression (\textbf{MLR}) \footnote{\url{https://mulan.sourceforge.net/datasets-mtr.html}}.
To obtain reliable results, we repeat algorithm evaluations 30 times on different dataset partitions of datasets, using 70\% of instances as training data and the rest as testing data.

\begin{table}[t]
  \caption{Compared algorithms for vector-valued output learning.}
  \label{tab.algorithms}
  \centering
  \small
  \setlength\extrarowheight{5pt}
  \begin{tabular}{l|l}
      \toprule
      Parameters                     & Algorithms               \\ \hline
      $\tau_I = 0 , \tau_S = 0$   & SRM-VV \cite{koltchinskii2001some} \\ \hline
      $\tau_I = 0 , \tau_S > 0$   & LRC-VV \cite{yu2014large,xu2016local,li2018multi}             \\ \hline
      $\tau_I > 0 , \tau_S = 0$   & SS-VV \cite{li2015semi,mojoo2017deep}               \\ \hline
      $\tau_I > 0 , \tau_S > 0$   & \texttt{LSVV}                      \\ 
      \bottomrule
  \end{tabular}
\end{table}
\noindent\textbf{Compared Methods.}
To verify theoretical findings in the linear space and kernel space, we conduct all algorithms in both two settings: For the linear space, we simply use $\phi(\xx)=\xx$ where the feature space is $\mathcal{S} = \mathbb{R}^d$; For the kernel space, as described in Section \ref{section.rf}, we adopt random Fourier features to approximate kernel hypotheses.
To improve scalability of algorithms, only a few random features are used $(D=100).$
Random Fourier features are often used to approximate shift-invariant kernels \cite{rahimi2007random}.
We apply random Fourier features as the nonlinear mappings, defined as
$
  \phi(\xx) = \sqrt{\frac{2}{D}} \cos (\boldsymbol{\Omega}^T\xx + \boldsymbol{b}),
$
 where $\OO = [\oo_1, \cdots, \oo_D] \in \mathbb{R}^{d \times D}$ is a frequency matrix drawn i.i.d. from the Gaussian distribution $\mathcal{N}(0, 1/\sigma^2)$ and $\boldsymbol{b} \in \mathbb{R}^D$ is drawn from a uniform distribution $[0, 2\pi].$

We compare {LSVV} to its special cases with various parameter settings for $\tau_I$ and $\tau_S$, listed in Table \ref{tab.algorithms}.
\begin{itemize}
  \item [1)] \textbf{SRM-VV}: solves the empirical risk minimization with regularization (SRM).
  This method has been presented for special cases of vector-valued learning.
  \item [2)] \textbf{LRC-VV}: solves SRM together with minimizing local Rademacher complexity.
  It was first proposed for multi-class \cite{li2018multi} and for multi-label learning \cite{xu2016local}.
  \item [3)] \textbf{SS-VV}: corresponds to manifold regularization on SRM, which was introduced into multi-class classification \cite{li2015semi} and multi-label learning \cite{mojoo2017deep}.
  \item [4)] \textbf{\texttt{LSVV}}: is the proposed algorithm, as shown in (\ref{eq.alg-obj}), which makes use of both local Rademacher complexity and manifold regularization.
\end{itemize}

\begin{table}[t]
  \caption{Comparison of average test error (\%) among \textbf{linear estimators} for \textbf{multi-class classification}.
  For each dataset, we bold the optimal test error and underline results which show no significant difference from the optimal one.}
  \label{tab.linear-average-error-mc}
  \centering
  \small
  \setlength\extrarowheight{3pt}
  \begin{tabular}{c|rrrr}
    \hline
                              & SRM-VV                         & SS-VV                                            & LRC-VV                     & \texttt{LSVV}           \\ \hline
    iris                      & 29.78$\pm$6.21                    & 28.89$\pm$4.16                                   & \underline{28.44$\pm$7.10} & \textbf{28.40$\pm$5.53} \\
    wine                      & 9.63$\pm$3.56                     & 8.89$\pm$5.62                                    & \underline{6.30$\pm$3.10}  & \textbf{5.93$\pm$4.61}  \\
    glass                     & 53.54$\pm$5.90                    & 51.38$\pm$13.61                                  & 52.92$\pm$3.37             & \textbf{47.69$\pm$6.62} \\
    svmguide2                 & 39.32$\pm$4.30                    & \underline{36.27$\pm$8.79}                       & 38.98$\pm$7.39             & \textbf{35.25$\pm$5.45} \\
    vowel                     & 74.72$\pm$3.28                    & 74.72$\pm$3.19                                   & 74.72$\pm$6.53             & \textbf{69.81$\pm$3.42} \\
    vehicle                   & 55.43$\pm$4.46                    & 54.41$\pm$9.40                                   & 55.20$\pm$6.95             & \textbf{49.45$\pm$3.39} \\
    segment                   & 17.49$\pm$4.79                    & 16.54$\pm$2.52                                   & 16.62$\pm$2.28             & \textbf{14.40$\pm$1.61} \\
    satimage                  & 21.19$\pm$3.47                    & 20.95$\pm$1.26                                   & 20.78$\pm$2.76             & \textbf{19.97$\pm$1.41} \\
    pendigits                 & 11.85$\pm$1.01                    & 11.77$\pm$1.42                                   & 11.29$\pm$1.26             & \textbf{10.30$\pm$1.40} \\
    letter                    & 48.49$\pm$4.88                    & 48.12$\pm$2.33                                   & 48.22$\pm$2.90             & \textbf{44.20$\pm$2.90} \\
    poker                     & 51.56$\pm$3.46                    & 50.67$\pm$1.34                                   & 51.46$\pm$3.55             & \textbf{49.83$\pm$0.47} \\
    shuttle                   & 6.86$\pm$2.00                     & 6.73$\pm$1.98                                    & 6.19$\pm$1.53              & \textbf{5.54$\pm$1.77}  \\
    Sensorless                & 48.49$\pm$4.48                    & 47.08$\pm$6.86                                   & 47.33$\pm$4.28             & \textbf{45.07$\pm$2.46} \\
    MNIST                     & 17.58$\pm$0.25                    & 17.49$\pm$0.27                                   & 17.52$\pm$0.27             & \textbf{17.23$\pm$0.34} \\
    connect-4                 & 34.18$\pm$0.23                    & 34.18$\pm$0.22                                   & 34.18$\pm$0.23             & \textbf{33.73$\pm$0.43} \\
    acoustic                  & 35.25$\pm$1.33                    & 35.18$\pm$2.45                                   & \underline{34.03$\pm$1.25} & \textbf{33.84$\pm$1.58} \\
    covtype                   & 27.52$\pm$2.10                    & 26.44$\pm$1.10                                   & 26.80$\pm$2.93             & \textbf{25.31$\pm$1.21} \\
    \hline
  \end{tabular}
\end{table}

\noindent\textbf{Experimental Settings.}
We construct the similarity matrix $\SS$ using a $10$-NN graph with binary weights, which are more efficient than the heat kernel weights used in our previous work \cite{li2019multi}.
The graph Laplacian is given by $\boldsymbol{L} = \boldsymbol{D} - \SS,$ where $\boldsymbol{D}$ is a diagonal matrix with $\boldsymbol{D}_{ii} = \sum_{j=1}^{n + u} \SS_{ij}.$
The predictive ability of \texttt{LSVV} is highly dependent on parameters $\tau_A, \tau_I, \tau_S$ and the Gaussian kernel parameter $\sigma$.
The candidate sets consist of the complexity parameter $\tau_A \in \{10^{-15}, 10^{-14}, \cdots, 10^{-6}\},$ unlabeled samples parameter $\tau_I \in \{0, 10^{-15}, 10^{-14}, \cdots, 10^{-6}\},$ the parameter for local Rademacher complexity term $\tau_S \in \{0, 10^{-10}, 10^{-9}, \cdots, 10^{-1}\},$ and tail parameter $\theta \in \{0, 0.1, \cdots, 0.9\}\times\min(K, D).$
For random Fourier features approaches, the Gaussian kernel parameter $\sigma$ is selected from the candidate $[2^{-5}, 2^{-4}, \cdots, 2^{5}].$
In semi-supervised settings, the inputs of both labeled examples and unlabeled examples are drawn from the marginal distribution $\rho_X$, but some of them are missing labels.
To simulate the semi-supervised setting, we randomly select some examples as unlabeled examples by ignoring their labels.
Specifically, we devise the semi-supervised settings by splitting the dataset uniformly that 70\% data for training and 30\% data for testing.
We repeat the data partitions and conduct compared algorithms $30$ times.
Since we split the dataset randomly and the data partitions are repeated $30$ times, the marginal distribution $\rho_X$ for every dataset is retained for semi-supervised learning.
For fair comparisons, we tune parameters to achieve optimal empirical performance for all algorithms via $5$-fold cross-validation on the labeled data for all datasets.

\subsection{Evaluations for Multi-class Classification}
\begin{table}[t]
  \caption{Comparison of average test error (\%) among \textbf{kernel estimators} for \textbf{multi-class classification}.
  For each dataset, we bold the optimal test error and underline results which show no significant difference from the optimal one.
    }
  \label{tab.kernel-average-error-mc}
  \centering
  \small
  \setlength\extrarowheight{1.8pt}
  \begin{tabular}{c|rrrr}
    \hline
    Datasets                  & SRM-VV        & SS-VV                     & LRC-VV                     & \texttt{LSVV}           \\ \hline
    iris                      & 7.56$\pm$5.12    & 7.56$\pm$3.72             & 7.11$\pm$3.65              & \textbf{4.44$\pm$3.51}  \\
    wine                      & 8.15$\pm$3.10    & \underline{6.67$\pm$4.83} & \underline{7.78$\pm$4.61}  & \textbf{5.56$\pm$5.56}  \\
    glass                     & 44.31$\pm$7.80   & 43.08$\pm$6.88            & 44.31$\pm$13.21            & \textbf{37.85$\pm$9.27} \\
    svmguide2                 & 26.10$\pm$2.96   & 25.93$\pm$3.67            & 25.59$\pm$4.77             & \textbf{24.07$\pm$2.04} \\
    vowel                     & 63.65$\pm$2.01   & 60.13$\pm$3.52            & 63.14$\pm$3.23             & \textbf{57.61$\pm$3.66} \\
    vehicle                   & 45.67$\pm$2.43   & 44.41$\pm$4.97            & 44.80$\pm$3.68             & \textbf{41.50$\pm$2.31} \\
    segment                   & 13.68$\pm$1.09   & 13.56$\pm$1.17            & 13.45$\pm$2.18             & \textbf{12.32$\pm$1.24} \\
    satimage                  & 15.27$\pm$0.68   & 15.16$\pm$0.94            & 15.10$\pm$1.84             & \textbf{14.27$\pm$0.52} \\
    pendigits                 & 6.99$\pm$1.07    & 6.91$\pm$0.93             & 6.95$\pm$0.45              & \textbf{6.01$\pm$1.27}  \\
    letter                    & 33.32$\pm$0.60   & 33.14$\pm$1.94            & 33.24$\pm$1.82             & \textbf{30.08$\pm$1.27} \\
    poker                     & 49.94$\pm$0.59   & 49.55$\pm$0.68            & 49.63$\pm$0.62             & \textbf{49.19$\pm$0.55} \\
    shuttle                   & 1.40$\pm$0.34    & 1.26$\pm$0.35             & 1.39$\pm$0.26              & \textbf{1.15$\pm$0.16}  \\
    Sensorless                & 13.31$\pm$0.62   & 13.22$\pm$0.25            & 13.20$\pm$0.41             & \textbf{11.55$\pm$0.28} \\
    MNIST                     & 12.62$\pm$0.27   & 12.57$\pm$0.22            & \underline{12.46$\pm$0.47} & \textbf{12.42$\pm$0.33} \\
    connect-4                 & 30.24$\pm$1.23   & 30.10$\pm$1.65            & 30.02$\pm$0.92             & \textbf{29.03$\pm$1.25} \\
    acoustic                  & {31.45$\pm$0.56} & {31.27$\pm$0.41}          & {31.44$\pm$0.68}           & \textbf{30.23$\pm$0.76} \\
    covtype                   & 24.24$\pm$0.32   & 24.23$\pm$0.26            & {24.18$\pm$0.21}           & \textbf{24.09$\pm$0.28} \\
    \hline
  \end{tabular}
\end{table}

Compared to our previous work \cite{li2018multi,li2019multi}, here we conduct experiments on larger datasets with fewer labeled examples.

In each partition, we uniformly sample 10\% of training samples as labeled pairs while the remaining 90\% are used as unlabeled instances.
Further, the multiple test errors were obtained to allow the statistical significance of the difference between each method and the optimal result.
We adopt a 95\% significance level in Table \ref{tab.linear-average-error-mc}, \ref{tab.kernel-average-error-mc}. 
For each dataset, we bold the optimal test error and underline results that show no significant difference from the optimal one.

The results in Table \ref{tab.linear-average-error-mc} and Table \ref{tab.kernel-average-error-mc} show:
(1) Our method outperforms the others on all datasets, both in the linear space and appropriate kernel space.
(2) As a classical margin-based multi-class classification model, SRM-VV shows the highest average test errors on all datasets, for both the linear estimator and approximate kernel estimator.
This is likely because it does not use local Rademacher complexity and ignores valuable information from unlabeled data.
(3) SS-VV makes use of unlabeled instances.
LRC-VV minimizes the tail sum of singular values together with the empirical loss and a penalty term for model complexity.
SS-VV or LRC-VV only utilize one additional regularizer and obtain comparable performance, being better than SRM-VV but worse than \texttt{LSVV.}
(4) Approximate kernel approaches always provide better results than linear approaches.
(5) Even using only a small number of random features $D=100,$ the results of approximate approaches are significantly better than linear estimators on \textit{iris}, \textit{pendigits}, \textit{shuttle}, \textit{Sensorless} and \textit{MNIST}.

\begin{table}[t]
  \centering
  \caption{Comparison between predictive performance of \textbf{linear estimators} for \textbf{multi-label learning}.
    The first four datasets are from MLC tasks, where performance indicator is the error rate $(\%)$ by the averaged Hamming loss.
    The last two datasets are MLR where performance indicator is the RMSE.
    }
  \label{tab.linear-average-error-ml}
  \small
  \setlength\extrarowheight{1.8pt}
  \begin{tabular}{c|llll}
    \hline
    Dataset                   & SRM-VV                         & SS-VV                                            & LRC-VV                     & \texttt{LSVV}            \\ \hline
    scene                     & 31.24$\pm$4.66                    & 30.98$\pm$9.64                                   & 16.99$\pm$0.50             & \textbf{16.76$\pm$0.35}  \\
    yeast                     & 30.29$\pm$3.39                    & 30.08$\pm$0.26                                   & 27.42$\pm$1.92             & \textbf{24.57$\pm$1.57}  \\
    corel5k                   & 28.32$\pm$3.29                    & 26.48$\pm$3.20                                   & 15.14$\pm$2.94             & \textbf{1.77$\pm$0.50}   \\
    bibtex                    & 41.72$\pm$4.03                    & 41.24$\pm$1.52                                   & 20.90$\pm$43.35            & \textbf{20.89$\pm$43.35} \\
    \hline
    rf2                       & 13.42$\pm$2.32                    & 12.13$\pm$0.87                                   & \underline{11.93$\pm$0.31} & \textbf{11.93$\pm$0.15}  \\
    scm1d                     & 19.81$\pm$12.22                   & 15.39$\pm$3.41                                   & 18.78$\pm$24.00            & \textbf{6.46$\pm$3.91}   \\
    \bottomrule
  \end{tabular}
\end{table}

\begin{table}[t]
  \centering
  \caption{Comparison between predictive performance of \textbf{kernel estimators} for \textbf{multi-label learning}.
  The first four datasets are from MLC tasks, where performance indicator is the error rate $(\%)$ by the averaged Hamming loss.
  The last two datasets are MLR where performance indicator is the RMSE.
  }
  \label{tab.kernel-average-error-ml}
  \small
  \begin{tabular}{c|llll}
    \hline
    Dataset                   & SRM-VV       & SS-VV           & LRC-VV         & \texttt{LSVV}           \\ \hline
    scene                     & 14.99$\pm$0.90  & 14.74$\pm$0.41  & 14.80$\pm$0.89 & \textbf{13.46$\pm$0.83} \\
    yeast                     & 23.05$\pm$0.75  & 22.77$\pm$0.39  & 22.63$\pm$0.45 & \textbf{22.32$\pm$0.47} \\
    corel5k                   & 1.01$\pm$0.03   & 0.98$\pm$0.01   & 0.95$\pm$0.02  & \textbf{0.94$\pm$0.00}  \\
    bibtex                    & {1.48$\pm$0.01} & {1.47$\pm$0.03} & 1.45$\pm$0.02  & \textbf{1.44$\pm$0.04}  \\
    \hline
    rf2                       & 1.19$\pm$0.09   & 1.14$\pm$0.04   & 1.03$\pm$0.07  & \textbf{0.80$\pm$0.05}  \\
    scm1d                     & 0.68$\pm$0.01   & 0.63$\pm$0.03   & 0.68$\pm$0.02  & \textbf{0.53$\pm$0.03}  \\
    \bottomrule
  \end{tabular}
\end{table}

\subsection{Evaluations for Multi-label Learning}
\label{sec.ml-exp}
To study performance on multi-label learning, four multi-label classification tasks and two multi-label regression tasks are utilized.
Labels for multi-label classification samples consist of a series of binary classifications.
We scale all labels of multi-label regression datasets to $[0, 1].$
For multi-label tasks, we consider the case where 50\% of labels are missing for training datasets to validate the efficiency of the Laplacian regularization.
Given a test set $(\xx_i, ~ \yy_i)_{i=1}^{n_t},$ we use the averaged Hamming loss as the criteria to evaluate the generalization performance of multi-label classification:
$
 \text{Error} = \frac{1}{n_tK} \sum_{i=1}^{n_t} \sum_{k=1}^K y_{ik}' \oplus y_{ik}.
$
Here, $\yy_i' = \mathbf{1}_{h(\xx_i) > 0.5}$ and $\oplus$ is the XOR operator.
For the multi-label regression, we employ the averaged root-mean-square error (RMSE):
$
 \text{Error} = \frac{1}{n_tK} \sum_{i=1}^{n_t} \|h(\xx_i) - \yy_i\|_2.
$

Results in Table \ref{tab.linear-average-error-ml} and Table \ref{tab.kernel-average-error-ml} show that
(1) The proposed \texttt{LSVV}
with 100 random features always provides the best empirical results among all eight approaches.
(2) Both Laplacian regularization (SS-VV) and the tail sum of singular values of $\WW$ (LRC-VV) outperform the structural risk minimization method (SRM-VV) on most datasets.

\begin{figure}[t]
  \centering
  \includegraphics[width=0.7\linewidth]{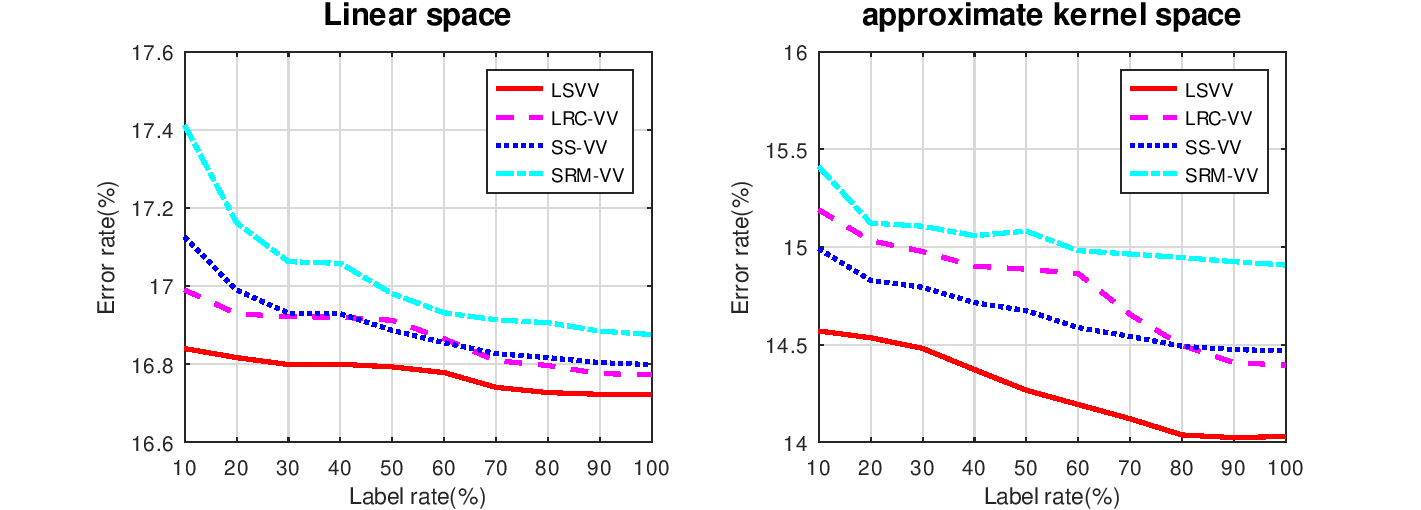}
  \caption{Averaged error rates w.r.t. difference labeled proportions on \textit{scene} dataset.}
  \label{fig.unlabeled}
\end{figure}
\subsection{Impact of the Laplacian Regularization}
In this section, we explore the impact of the Laplacian regularization term by varying the proportion of labeled examples.
We perform all compared algorithms in both the input space and approximate kernel space on the \textit{Scene} dataset.
For every labeled proportion, we divide the \textit{Scene} dataset into labeled and unlabeled datasets uniformly.
We repeat the data partitions and compared experiments $10$ times and record the average error rates in Figure \ref{fig.unlabeled}.
Both \texttt{LSVV} and SS-VV are the semi-supervised algorithms that can leverage the unlabeled examples, while LRC-VV and SRM-VV only train the model on the labeled data.
As shown in Figure \ref{fig.unlabeled},
the test errors of all methods decrease as the number of labeled samples increases, but the semi-supervised algorithm (\texttt{LSVV} and SS-VV) outperform the supervised ones (LRC-VV and SRM-VV).

\subsection{Influence of the Threshold $\theta$}
\begin{figure}[t]
 \centering
 \subfigure[corel5k]{
 \includegraphics[width=0.7\linewidth]{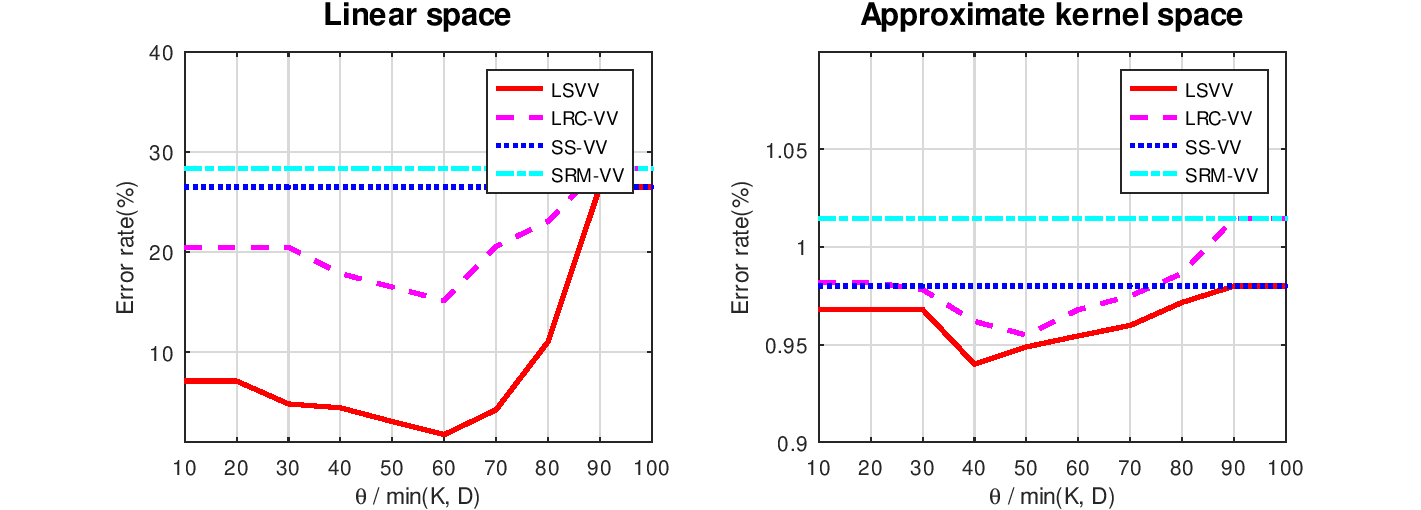}%
 }
 \subfigure[bibtex]{
 \includegraphics[width=0.7\linewidth]{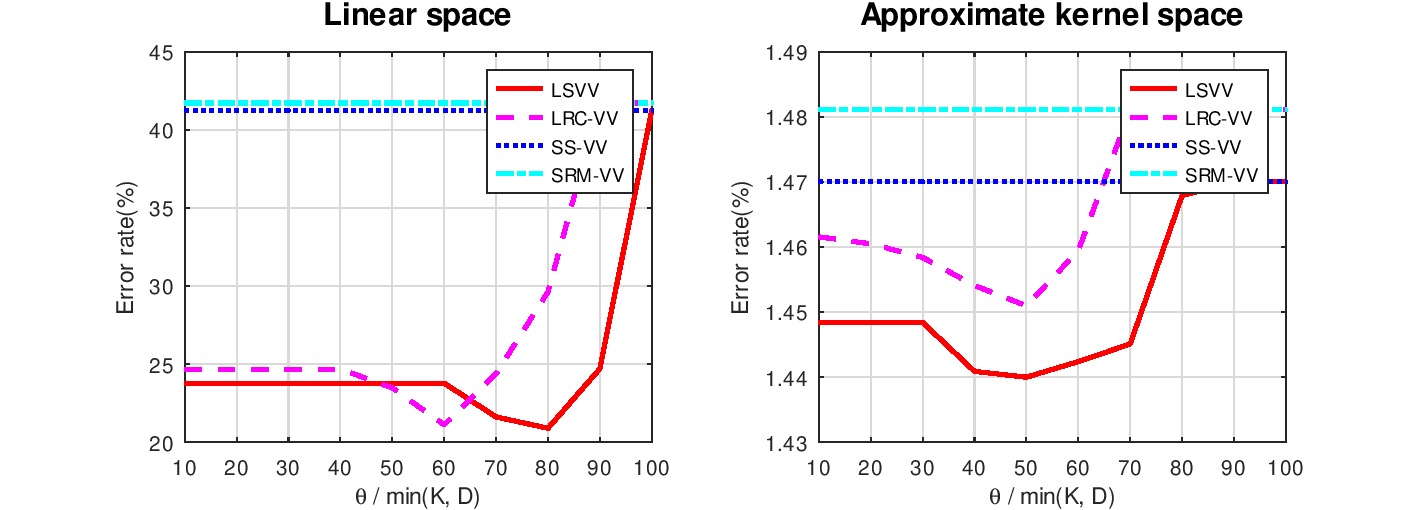}%
 }
 \caption{Averaged error rates w.r.t. the threshold $\theta$ on \textit{corel5k} and \textit{bibtex} datasets.}
 \label{fig.theta}
\end{figure}

We vary the thresholding value $\theta$ for the tail sum of eigenvalues, to determine the importance of an appropriate threshold $\theta$.
If $\theta/\min(K, D) = 1$, there is no constraint on the eigenvalues of the estimator, and thus the proposed algorithm \texttt{LSVV} degrades into SS-VV and LRC-VV degrades into SRM-VV, respectively.
When $\theta = 0$, the tail sum of singular values becomes the trace norm, which corresponds to the \textit{global} Rademacher complexity.
On two representative datasets, we run all compared methods w.r.t the varying threshold $\theta$.

The results are reported in Figure \ref{fig.theta}.
The performance of \texttt{LSVV} is limited when $\theta$ is large and leads to the same performance as that of SS-VV.
As $\theta$ becomes smaller, the Laplacian regularization term considers more eigenvalues that reduce the error rate.
When $\theta=0$, \texttt{LSVV} and LRC-VV use the trace norm of the estimator as the regularization, which leads to preferable performance but is not optimal.
As discussed in Remark \ref{re.appropriate-theta}, an appropriate threshold $\theta$ is key to obtaining a lower error bound.
There always exists an optimal thresholding $\theta$ for \texttt{LSVV} and LRC-VV, offering the lowest error rates.
The empirical results coincide with our theoretical findings (Corollary \ref{cor.lrc-bound-kernel} and Corollary \ref{cor.lrc-bound-linear}).

\section{Conclusion}
\label{sec.conclusion}
Based on our previous works for multi-class classification \cite{li2018multi,li2019multi}, we introduce local Rademacher complexity for vector-valued learning with unlabeled samples.
Compared to our previous work, we remove the strict assumption that the loss function should be smooth and unify the generalization analysis of vector-valued learning for both linear and kernel hypotheses.
Besides, we devise an efficient semi-supervised vector-valued learning with random features to approximate kernel function, which achieves a good tradeoff between computational efficiency and predictive performance.
However, there are still some drawbacks in this paper: 1) The theoretical results cannot directly apply to outputs with more complicated structures, i.e. structured prediction; 2) Our analysis focused on the generalization analysis for kernel methods, ignoring the optimization problems with stochastic gradient descent methods; 3) Our generalization bounds only apply to kernel methods and linear models but fail to apply to deep neural networks.
These limitations motivate us to extend our results to structured prediction, stochastic gradient descent, and deep neural networks in the future.

\section*{Acknowledgment} 

This work was supported in part by Excellent Talents Program of Institute of Information Engineering, CAS, Special Research Assistant Project of CAS (No. E0YY231114), Beijing Outstanding Young Scientist Program (No. BJJWZYJH012019100020098), and National Natural Science Foundation of China (No. 62076234, No. 62106257).

\appendix
\section{Proof}
\subsection{Proof of Theorem \ref{thm.general-lrc-bound}}
\begin{proposition}[First part of Theorem 3.3 in \cite{bartlett2005local}]
  \label{prop.primal-lrc-bound}
    Let $\mathcal{Z}$ be any set, $({z}_1, \cdots, {z}_m) \in \mathcal{Z}^m.$
    For a class of bounded functions $\mathcal{F}: \mathcal{Z} \to \mathbb{R}$ with ranges in $[a, a'].$
    Assume there is some function $T: \mathcal{F} \to \mathbb{R}^+$ and some constant $\alpha$ such that for any $f \in \mathcal{F},$ $\text{Var}(f) \leq T(f) \leq \alpha Pf.$
    Assume there is a sub-root function $\psi$ and a fixed point $r^*$ of $\psi,$ for any $r \geq r^*$ satisfying
    \begin{align}
      \label{eq.primal-lrc-bound.eq1}
      \psi(r) \geq \alpha ~ \mathcal{R} \left(\left\{f \in \mathcal{F}: T(f) \leq r \right\}\right).
    \end{align}
    For any $K > 1$ and any $\delta \in (0,1),$ with a probability of at least $1 - \delta,$
    \begin{equation}
      \begin{aligned}
        \label{eq.primal-lrc-bound.eq2}
        Pf ~ \leq ~ \frac{K}{K-1} \widehat{P} f + c_1 r^*
         ~ + ~ c_2 \frac{\log (1/\delta)}{m} ,
      \end{aligned}
    \end{equation}
    where $c_1 = \frac{704K}{\alpha},$ $c_2 = 11(a'-a)+26K\alpha,$ $P f = \mathbb{E} [f({z})]$ is the expectation and $\widehat{P} f = \frac{1}{m} \sum_{i=1}^m f(\xx_i)$ is the corresponding empirical version on $\mathcal{Z}^m.$
\end{proposition}

Based on Proposition \ref{prop.primal-lrc-bound}, we prove Theorem \ref{thm.general-lrc-bound} as follows.

\begin{proof}\textbf{of Theorem \ref{thm.general-lrc-bound}.}
  According to Proposition \ref{prop.primal-lrc-bound}, we set $$f = \ell(\widehat{h}(\xx), \yy) - \ell(h^*(\xx), \yy)$$ for any $(\xx, \yy) \in \mathcal{X} \times \mathcal{Y}.$ 
  Here, $\widehat{h}\in \mathcal{H}$ corresponds to the estimator with the minimal empirical loss and $h^* \in \mathcal{H}$ is the estimator with the minimal expectation loss.
  Thus, we have $Pf = \mathcal{E}(\widehat{h}) - \mathcal{E}(h^*) .$

  \textbf{First step: (\ref{eq.primal-lrc-bound.eq2}) $\to$ (\ref{eq.lrc-bound.eq2}).}
  Since $\widehat{h}$ indicates the minimal empirical loss, we have $\widehat{P} f \leq \widehat{\mathcal{E}}(\widehat{h}) - \widehat{\mathcal{E}}(h^*) \leq 0.$
  Thus we just omit the $\widehat{P} f$ term in (\ref{eq.primal-lrc-bound.eq2}) and get
  \begin{equation}
    \label{eq.proof.lrc-bound.eq1}
    \begin{aligned}
      \mathcal{E}(\widehat{h}) - \mathcal{E}(h^*)
      \leq ~ c_1 r^* + c_2\frac{\log(1/\delta)}{n}.
    \end{aligned}
  \end{equation}
  Beacause the loss function $\ell$ is bounded in $[0, B],$ it holds $f \in [0, ~ B],$ $a' = B$ and $a = 0$ in Proposition \ref{prop.primal-lrc-bound}.
  However, $\mathcal{E}(h^*)$ is also the expected infimum in the loss space, so we have $Pf = \mathcal{E}(\widehat{h}) - \mathcal{E}(h^*) \in [0, B].$ The variance exists $\text{Var}(f) = Pf^2 - [Pf]^2 \leq Pf^2 \leq B Pf.$
  We set $T(f) = Pf^2$ and then have $\alpha = B.$
  By setting $\alpha = B,$ $a=0,$ $a'=B$ and a small $K > 1$ in (\ref{eq.proof.lrc-bound.eq1}), the equation (\ref{eq.lrc-bound.eq2}) is obtained.

  \textbf{Second step: (\ref{eq.primal-lrc-bound.eq1}) $\to$ (\ref{eq.lrc-bound.eq1}).}
  We consider the variance $T(f) = Pf^2.$ 
  Thus, $\mathcal{R} \left(\left\{f \in \mathcal{F}: T(f) \leq r \right\}\right)$ becomes $\mathcal{R} (\{\mathbb{E} [\ell_{\widehat{h}} - \ell_{h^*}]^2 \leq r \}),$ where $\ell_{h} = \ell(h(\xx), \yy)$ for any $\ell \in \mathcal{L}$ and $(\xx, \yy) \in \mathcal{X} \times \mathcal{Y},$ which concides with local Rademacher complexity $\mathcal{R}(\mathcal{H}_r)$ in Definition \ref{def.lrc-loss}.
  As such (\ref{eq.primal-lrc-bound.eq1}) requires a sub-root function $\psi_1$ satisfying
  \begin{align}
    \label{eq.proof.lrc-bound.eq2}
    \psi_1(r) \geq B \mathcal{R}(\mathcal{L}_r).
  \end{align}
  Using the contraction property in Lemma \ref{lem.contraction}, we have
  \begin{align}
    \label{eq.proof.lrc-bound.eq3}
    \sqrt{2} B L \mathcal{R}(\mathcal{H}_r) \geq B \mathcal{R}(\mathcal{L}_r).
  \end{align}
  We consider a sub-root function $\psi(r)$ such that
  \begin{align}
    \label{eq.proof.lrc-bound.eq5}
    \psi(r) \geq \sqrt{2} B L \mathcal{R}(\mathcal{H}_r).
  \end{align}
  Combining (\ref{eq.proof.lrc-bound.eq3}) and (\ref{eq.proof.lrc-bound.eq5}), we then find that the sub-root function $\psi$ satisfies the condition (\ref{eq.proof.lrc-bound.eq2}) and we finally get the condition in (\ref{eq.lrc-bound.eq1}).
\end{proof}

\subsection{Proof of Theorem \ref{thm.estimate-lrc-kernel}}
\begin{proof} \textbf{of Theorem \ref{thm.estimate-lrc-kernel}.}
  Due to the contraction lemma, the symmetry of the $\epsilon_i$ and $\|h\|_2 \leq \|h\|_1,$ there exists
  \begin{align*}
    \mathcal{R}(\mathcal{H}_r) & = \mathcal{R}(h\in \mathcal{H}_p:\mathbb{E}[L^2\|{h}-h^*\|_2^2] \leq r)                     \\
                               & =\mathcal{R}(h - h^*:h\in \mathcal{H}_p, \mathbb{E}[\|{h}-h^*\|_2^2] \leq \frac{r}{L^2})    \\
                               & \leq \mathcal{R}(h - g:h, g\in \mathcal{H}_p, \mathbb{E}[\|{h}-g\|_2^2] \leq \frac{r}{L^2}) \\
                               & =2\mathcal{R}(h:h\in \mathcal{H}_p, \mathbb{E}[\|h\|_2] \leq \frac{\sqrt{r}}{2L})           \\
                               & \leq 2\mathcal{R}(h:h\in \mathcal{H}_p, \mathbb{E}[\|h\|_1] \leq \frac{\sqrt{r}}{2L}).
  \end{align*}
  Let $\|\WW\|_p=\|\WW\|_*.$
  We introduce a new hypothesis space $\mathcal{H}_{2,1},$ satisfying
  \begin{align}
    \label{eq.LRC_norm}
    \mathcal{R}(\mathcal{H}_r) \leq 2\mathcal{R}(\mathcal{H}_{2,1})
  \end{align}
  where
  \begin{align*}
    \mathcal{H}_{2,1}=\{\xx \to \WW^T\phi(\xx):
    \|\WW\|_* \leq 1, \mathbb{E}[\|{h}\|_1] \leq \frac{\sqrt{r}}{2L}\}.
  \end{align*}
  For any $\theta \in \mathbb{N},$ there holds for local Rademacher complexity
  \begin{equation}
    \label{eq.decomposition}
    \begin{split}
      &\frac{1}{n+u}\sum_{i=1}^{n + u}\sum_{k=1}^K \ee_{ik}\langle\WW_{\cdot k},\phi(\xx_i)\rangle
      =~\frac{1}{n+u}\sum_{k=1}^K\bigg\langle\WW_{\cdot k}, \sum_{i=1}^{n + u} \ee_{ik}\phi(\xx_i)\bigg\rangle\\
      =~&\sum_{k=1}^K \Bigg[\bigg\langle\sum_{j=1}^\theta \sqrt{\lambda_j}\langle\WW_{\cdot k}, \varphi_j\rangle\varphi_j, \sum_{j=1}^\theta\frac{1}{\sqrt{\lambda_j}}\bigg\langle\frac{1}{n+u}\sum_{i=1}^{n + u}\ee_{ik}\phi(\xx_i), \varphi_j\bigg\rangle\varphi_j\bigg\rangle \\
      &\quad+\bigg\langle\WW_{\cdot k}, \sum_{j>\theta}\bigg\langle\frac{1}{n+u}\sum_{i=1}^{n + u}\ee_{ik}\phi(\xx_i),\varphi_j\bigg\rangle\varphi_j\bigg\rangle\Bigg].\\
    \end{split}
  \end{equation}

  Following the proof of Theorem 3 \cite{cortes2013learning}, the Cauchy-Schwarz inequality and Jensen's inequality, the above inequality (\ref{eq.decomposition}) yeilds
  \begin{equation}
    \label{eq.LRC-estimate}
    \begin{split}
      &\mathcal{R}(\mathcal{H}_{2,1}) 
      =~ \mathbb{E} ~ \left[\sup_{h\in\mathcal{H}_{2,1}}\sum_{i=1}^{n + u}\sum_{k=1}^K
      \ee_{ik}\big\langle\WW_{\cdot k},\phi(\xx_i)\big\rangle\right]\\
      \leq~ &\sup_{h\in\mathcal{H}_{2,1}} \sum_{k=1}^K \sqrt{\Bigg(\sum_{j=1}^\theta\lambda_j\langle\WW_{\cdot k},\varphi_j\rangle^2\Bigg)
      \Bigg(\frac{1}{n+u}\sum_{j=1}^\theta\frac{1}{\lambda_j}\mathbb{E} ~ [\langle \phi(\xx), \varphi_j\rangle^2]\Bigg)} \\
      &+ \|\WW\|_*
      \sqrt{\frac{1}{n+u}\sum_{j>\theta}\mathbb{E} ~ [\langle \phi(\xx), \varphi_j\rangle^2]}.
    \end{split}
   \end{equation}
  By the eigenvalue decomposition, it holds
  $\mathbb{E}[|h_y|]=\sqrt{\sum_{j=1}^\infty\lambda_j\langle\WW_{\cdot k},\varphi_j\rangle^2}$
  such that
  \begin{align}
    \label{eq.eigenvalue_decomposition}
    \sum_{k=1}^K\sqrt{\sum_{j=1}^\theta\lambda_j\langle\WW_{\cdot k},\varphi_j\rangle^2}
      \leq \mathbb{E}[\|h\|_1] \leq \frac{\sqrt{r}}{2L},
  \end{align}
  and 
  \begin{equation}
    \label{eq.expected_Pi}
    \begin{aligned}
      \mathbb{E}[\langle\phi(\xx), \varphi_j\rangle^2] = \lambda_j
    \end{aligned}
  \end{equation}
  Substituting \eqref{eq.eigenvalue_decomposition} and \eqref{eq.expected_Pi} into \eqref{eq.LRC-estimate} and using $\|\WW\|_* \leq 1$, we have
  \begin{equation}
    \begin{split}
      \mathcal{R}(\mathcal{H}_r) 
      ~\leq~ 2\mathcal{R}(\mathcal{H}_{2,1})
      ~\leq~ \min_{0 \leq \theta \leq 0+u}
      \frac{1}{L}\sqrt{\frac{\theta r}{n+u}}
      +~ 2\sqrt{\sum_{j=\theta+1}^{n+u}\frac{\lambda_j}{n+u}}.
    \end{split}
  \end{equation}
  Applying the above result into \eqref{eq.LRC_norm}, we complete the proof.
\end{proof}

\subsection{Proof of Theorem \ref{thm.estimate-lrc-linear}}
\begin{proof}\textbf{of Theorem \ref{thm.estimate-lrc-linear}.}
  Due to the contraction lemma, the symmetry of Rademacher variables and $\mathbb{E} [\phi(\xx)^T\phi(\xx)] \leq 1,$ we have
  \begin{equation}
    \begin{aligned}
      \label{thm.estimate-lrc-linear.eq-00}
      &\mathcal{R}(\mathcal{H}_r)\\
      = ~ &\mathcal{R}\big(h\in \mathcal{H}_p:\mathbb{E}[L^2\|{h}-h^*\|_2^2] \leq r\big)\\
      = ~ &\mathcal{R}\big(h - h^*:h\in \mathcal{H}_p, \mathbb{E}[\|{h}-h^*\|_2^2] \leq \frac{r}{L^2}\big) \\
      \leq ~ &\mathcal{R}\big(h - g:h, g\in \mathcal{H}_p, \mathbb{E}[\|{h}-g\|_2^2] \leq \frac{r}{L^2}\big) \\
      = ~ &2 ~\mathcal{R}\big(h:h\in \mathcal{H}_p, \mathbb{E}[\|h\|_2^2] \leq \frac{r}{4L^2}\big)\\
      = ~ &2 ~\mathcal{R}\big(h:h\in \mathcal{H}_p, \mathbb{E}[\phi(\xx)^T\WW\WW^T\phi(\xx)] \leq \frac{r}{4L^2}\big)\\
      = ~ &2 ~\mathcal{R}\big(h:h\in \mathcal{H}_p, \mathbb{E}[\|\WW\WW^T\|] \leq \frac{\sqrt{r}}{2L}\big) \\
      = ~ &2 ~\mathcal{R}\big(\mathcal{H}_r^{\WW}\big).
    \end{aligned}
  \end{equation}
  The inequalities above provide a constraint on $\WW$ that is $\mathbb{E}[\|\WW\WW^T\|] \leq \frac{\sqrt{r}}{2L}$, which is useful when we reduce terms related to $\WW$.
  Then, local Rademacher complexity $\mathcal{R}\big(\mathcal{H}_r^{\WW}\big)$ can be rewritten as
  \begin{equation}
    \label{thm.estimate-lrc-linear.eq-0}
    \begin{aligned}
      &\mathcal{R}\big(\mathcal{H}_r^{\WW}\big)\\
      = ~&\mathbb{E}
      \left[\sup_{h \in\mathcal{H}_r^{\WW}}
      \frac{1}{n + u}\sum_{i=1}^{n + u} \sum_{k=1}^K \ee_{ik} h_j(\xx_i) \right] \\
      = ~&\mathbb{E}
      \left[\sup_{h \in\mathcal{H}_r^{\WW}}
      \frac{1}{n + u}\sum_{i=1}^{n + u} \sum_{k=1}^K \ee_{ik} \WW_{\cdot j}^T \phi(\xx_i) \right] \\
      = ~&\mathbb{E}
      \left[\sup_{h \in\mathcal{H}_r^{\WW}}
      \sum_{k=1}^K \WW_{\cdot j}^T \left(\frac{1}{n + u}\sum_{i=1}^{n + u} \ee_{ik} \phi(\xx_i)\right) \right] \\
      = ~&\mathbb{E}
      \left[\sup_{h \in\mathcal{H}_r^{\WW}}
      \left\langle\WW, \XX_{\ee}\right\rangle
      \right],
    \end{aligned}
  \end{equation}
  where $\WW, \XX_{\ee} \in \mathbb{R}^{D \times K}$ and $\left\langle\WW, \XX_{\ee}\right\rangle = \text{Tr}\big(\WW^T\XX_{\ee}\big)$ represents the trace norm.
  We define the matrix $\XX_{\ee}$ as follows:
  \begin{align*}
      \XX_{\ee} := \left[\frac{1}{n + u}
      \sum_{i=1}^{n + u} \ee_{i1}\phi(\xx_i),
          \cdots, \frac{1}{n + u}
          \sum_{i=1}^{n + u} \ee_{iK}\phi(\xx_i)\right].
  \end{align*}
  Borrowing the proof sketches of Thereom 5 in \cite{xu2016local},
  we consider the SVD decomposition
  \begin{align*}
    \WW = \sum_{j \geq 1} \boldsymbol{u}_j \boldsymbol{v}_j^T \tilde{\lambda}_j,
  \end{align*}
  where $\boldsymbol{u}_j$ and $\boldsymbol{v}_j$ are the orthogonal vectors.
  It holds the following inequalities 
  \begin{align*}
    &\left\langle \WW, \XX_{\ee} \right\rangle \\
    \leq ~ &\sum_{j=1}^\theta \left\langle \boldsymbol{u}_j \boldsymbol{v}_j^T \tilde{\lambda}_j, \XX_{\ee} \boldsymbol{u}_j \boldsymbol{u}_j^T \right\rangle
    + \sum_{j>\theta}\left\langle \WW, \XX_{\ee} \boldsymbol{u}_j \boldsymbol{u}_j^T \right\rangle \\
    \leq ~ &\left\langle \sum_{j=1}^\theta \boldsymbol{u}_j \boldsymbol{v}_j^T \tilde{\lambda}_j, \sum_{j=1}^\theta \XX_{\ee} \boldsymbol{u}_j \boldsymbol{u}_j^T \right\rangle
    + \left\langle \WW, \sum_{j>\theta} \XX_{\ee} \boldsymbol{u}_j \boldsymbol{u}_j^T \right\rangle \\
    \leq ~ &\left\|\sum_{j=1}^\theta \boldsymbol{u}_j \boldsymbol{v}_j^T \tilde{\lambda}_j^2\right\| \left\|\sum_{j=1}^\theta \XX_{\ee} \boldsymbol{u}_j \boldsymbol{u}_j^T\tilde{\lambda}_j^{-1}\right\| 
    + \| \WW \|_* \left\|\sum_{j>\theta} \XX_{\ee} \boldsymbol{u}_j \boldsymbol{u}_j^T \right\|.
  \end{align*}
  Then, we begin to bound the norm terms in the above inequalities.
  According to the definition of $\mathcal{H}_r^{\WW},$ it holds that $\mathbb{E}[\|\WW\WW^T\|] \leq \frac{\sqrt{r}}{2L}.$
  Thus, we have
  \begin{equation}
    \label{thm.estimate-lrc-linear.eq-1}
    \begin{aligned}
      &\Big\|\sum_{j=1}^\theta \boldsymbol{u}_j \boldsymbol{v}_j^T \tilde{\lambda}_j^2\Big\|
      =\Big\|\sum_{j=1}^\theta \boldsymbol{u}_j \boldsymbol{u}_j^T \tilde{\lambda}_j^2\Big\| \\
      \leq ~ &\Big\|\sum_{j=1}^\infty \boldsymbol{u}_j \boldsymbol{u}_j^T \tilde{\lambda}_j^2\Big\|
      = \Big\|\mathbb{E}[\WW\WW^T]\Big\|
      \leq \frac{\sqrt{r}}{2L}.
    \end{aligned}
  \end{equation}
  Using the properties of SVD decomposition, there exists
  \begin{equation}
    \label{thm.estimate-lrc-linear.eq-2}
    \begin{aligned}
      \mathbb{E} \left[\Big\|\sum_{j=1}^\theta \XX_{\ee} \boldsymbol{u}_j \boldsymbol{u}_j^T\tilde{\lambda}_j^{-1}\Big\|\right]
      = \mathbb{E} \left[\sqrt{\sum_{j=1}^\theta \tilde{\lambda}_j^{-2} \langle \XX_{\ee}, \boldsymbol{u}_j \rangle^2}\right] 
      \leq ~ \sqrt{\sum_{j=1}^\theta \frac{\tilde{\lambda}_j^{-2}}{n + u} \mathbb{E} [\langle \phi(\xx), \boldsymbol{u}_j \rangle^2] }
      ~ \leq ~ \sqrt{\frac{\theta}{n + u}}.
    \end{aligned}
  \end{equation}
  Then, we also have
  \begin{align}
    \label{thm.estimate-lrc-linear.eq-3}
    \mathbb{E} &\left[\Big\|\sum_{j>\theta} \XX_{\ee} \boldsymbol{u}_j \boldsymbol{u}_j^T \Big\|\right] \leq \sqrt{\frac{1}{n + u}\sum_{j > \theta} \tilde{\lambda}_j^2}.
  \end{align}
  We set the norm of $\|\WW\|_p$ in $\mathcal{H}_r^{\WW}$ as trace norm
  \begin{align}
    \label{thm.estimate-lrc-linear.eq-4}
    \|\WW\|_* \leq 1.
  \end{align}
  Substituting (\ref{thm.estimate-lrc-linear.eq-1}), (\ref{thm.estimate-lrc-linear.eq-2}), (\ref{thm.estimate-lrc-linear.eq-3}) and (\ref{thm.estimate-lrc-linear.eq-4}) into (\ref{thm.estimate-lrc-linear.eq-0}),
  we then have
  \begin{equation}
    \label{thm.estimate-lrc-linear.eq-5}
    \begin{aligned}
      \mathcal{R}\big(\mathcal{H}_r^{\WW}\big)
      = ~ \mathbb{E} \left[\sup_{h \in\mathcal{H}_r^{\WW}} \Big\langle\WW, \XX_{\ee}\Big\rangle \right] 
      \leq ~ \min_{0 \leq \theta} \frac{1}{2L}\sqrt{\frac{\theta r}{n + u}} + \sqrt{\frac{1}{n + u}\sum_{j > \theta} \tilde{\lambda}_j^2}.
    \end{aligned}
  \end{equation}
  Combining (\ref{thm.estimate-lrc-linear.eq-00}) and (\ref{thm.estimate-lrc-linear.eq-5}), we have
  \begin{equation}
    \begin{split}
      \mathcal{R}(\mathcal{H}_r)
      ~\leq~ 2\mathcal{R}\big(\mathcal{H}_r^{\WW}\big)
      ~\leq~ \min_{0 \leq \theta}
      \frac{1}{L}\sqrt{\frac{\theta r}{n+u}}
      +~ 2\sqrt{\sum_{j=\theta+1}^{n+u}\frac{\tilde{\lambda}_j^2}{n+u}}.
    \end{split}
  \end{equation}
  We complete the proof.
\end{proof}

\section{Optimization}
Inspired by generalized SVT methods, our previous work \cite{li2019multi} proposed a partly singular values thresholding learning framework based on the proximal gradient for the linear hypotheses.
To solve the minimization (\ref{eq.alg-obj}) in both the linear hypothesis space and approximate kernel hypothesis space, in this paper, we extend the previous algorithm with feature mappings and adaptive learning rates.
As illustrated in Algorithm \ref{alg.lsvt}, updating $\WW$ requires two steps:
(1) updating $\WW$ using mini-batch gradient descent on $g(\WW)$; 
(2) updating partial singular values to minimize $\sum_{j > \theta} \tilde{\lambda}_j(\WW).$

\subsubsection{Updating $\WW$ with Proximal Gradient Descent}
Consider the mini-batch gradient descent with $m$ samples,
\begin{align}
  \label{eq.first-step}
  \boldsymbol{Q}_t=\WW_t-\eta_t\nabla g(\WW_t),
\end{align}
where $\nabla g(\WW_t)$ is the derivative of differentiable terms
\begin{equation}
  \label{eq.gradient_of_g}
 \nabla g(\WW_t) = \frac{1}{m}\sum_{i=1}^{m} \frac{\partial ~\ell(h(\xx_i), ~ \yy_i)}{\partial~ \WW_t}
 + 2\tau_A \WW_t
 + 2\tau_I\widetilde{\mathbf{X}}\boldsymbol{L}\widetilde{\mathbf{X}}^T\WW_t.
\end{equation}
Here, $\eta_t$ is the learning rate for the $t$-iteration.
The term $\widetilde{\mathbf{X}}\boldsymbol{L}\widetilde{\mathbf{X}}^T$ is constant and time-consuming, thus we compute it before the iteration as shown in Algorithm \ref{alg.lsvt} (line 3).

\subsection{Updating $\WW$ with Singular Values Thresholding}
Consider singular value decomposition ${\boldsymbol U}{\boldsymbol \Sigma}{\boldsymbol V}^T = \boldsymbol{Q}_t,$ where ${\boldsymbol U} \in \mathbb{R}^{d \times d}$ and ${\boldsymbol V} \in \mathbb{R}^{K \times K}$ are orthogonal matrices, and ${\boldsymbol \Sigma}$ is diagonal with nonincreasing singular values
\begin{align}
  \label{eq.lsvt}
  \WW_{t+1}={\boldsymbol U}{\boldsymbol \Sigma}_\tau^\theta{\boldsymbol V}^T.
\end{align}
Here, with $\tau=\eta_t\tau_S,$ only first $\theta$ singular values are updated
\begin{align*}
  ({\boldsymbol \Sigma}_\tau^\theta)_{jj}=
  \left\{
  \begin{aligned}
       & |{\boldsymbol \Sigma}_{jj}-\tau|_+ & j \leq \theta, \\
       & {\boldsymbol \Sigma}_{jj},      & j > \theta.
  \end{aligned}
  \right.
\end{align*}

\subsection{Partly Singular Values Thresholding}
\label{sec.psvt}
In each iteration, to obtain a tight surrogate of Eq. (\ref{eq.alg-obj}), we keep $\tau_S \sum_{j > \theta} \lambda_j(\WW)$ while only relaxing $g(\WW)$, leading to a proximal regularization of $g(\WW)$ at $\WW_t$
\begin{align}
    \label{eq.proximal-gradient}
    \begin{split}
        \WW_{t+1}
        =~&\argmin_{\WW} ~ g(\WW_t) + \langle\nabla g(\WW_t),\WW-\WW_t \rangle
        ~+~\frac{1}{2\eta_t}\|\WW-\WW_t\|_F^2 +  \tau_S \sum_{j > \theta} \lambda_j(\WW)\\
        =~&\argmin_{\WW} ~ \frac{1}{2\eta_t}\|\WW-(\WW_t-\eta_t\nabla g(\WW_t)\|_F^2
        ~+~\tau_S \sum_{j > \theta} \lambda_j(\WW),
    \end{split}
\end{align}
where $\eta_t$ is the learning rate at the $t$-th iteration to update gradients, $\nabla g(\WW_t)$ is the derivative of $g(W)$ at $\WW_t$ and terms independent on $\WW$ are ignored.

\begin{proposition}[Theorem 6 of \cite{xu2016local}]
    \label{thm.prop-csvt}
    Let $\boldsymbol{Q}_t\in\mathbb{R}^{D \times K}$ with rank $r.$
    Its SVD decomposition is $\boldsymbol{Q}_t={\boldsymbol U}{\boldsymbol \Sigma}{\boldsymbol V}^T,$
    where ${\boldsymbol U} \in \mathbb{R}^{d \times r}$ and ${\boldsymbol V} \in \mathbb{R}^{K \times r}$ have orthogonal columns and ${\boldsymbol \Sigma}$ is diagonal. Then, it holds
    \begin{align}
      \label{eq.thm-svt}
        \mathcal{D}_\tau^\theta(\boldsymbol{Q}_t)=\argmin_{\WW}\left\{\frac{1}{2}\|\WW-\boldsymbol{Q}_t\|_F^2+\tau\sum_{j>\theta}\lambda_j(\WW)\right\},
    \end{align}
    is given by $\mathcal{D}_\tau^\theta={\boldsymbol U}{\boldsymbol \Sigma}_\tau^\theta{\boldsymbol V}^T,$
    where ${\boldsymbol \Sigma}_\tau^\theta$ is diagonal with
    \begin{align*}
        ({\boldsymbol \Sigma}_\tau^\theta)_{jj}=
        \left\{
        \begin{aligned}
             & |{\boldsymbol \Sigma}_{jj}-\tau|_+ & j \leq \theta, \\
             & {\boldsymbol \Sigma}_{jj},      & j > \theta.
        \end{aligned}
        \right.
    \end{align*}
\end{proposition}
Applied to Proposition \ref{thm.prop-csvt}, the proximal mapping (\ref{eq.proximal-gradient}) is equal to (\ref{eq.thm-svt}) with $\tau=\eta_t\tau_S,$ and then we get the result.

\section{Mini-batch Gradients for Specific Tasks}
In the inequality \eqref{eq.gradient_of_g}, only the derivative of the loss function $\frac{\partial ~\ell(h(\xx_i), ~ \yy_i)}{\partial~ \WW_t}$ needs to be determined.
In this section, we provide this derivative for two specific tasks: multi-class classification and multi-label learning.

\subsection{Multi-class Classification}
Consider a multi-class classification problem with $K$ classes.
The output space is written in the one-hot form $\mathcal{Y} = \{0, 1\}^K,$ which consists of a single $1$ and $K-1$ zeros.
$$\yy_i = [0, \cdots, 0, 1, 0, \cdots, 0]^T,$$
where only the $k$-th element is labeled as one.
The margin of multi-class classification is
\begin{align*}
 m_h(\xx_i, ~ \yy_i) = [h(\xx_i)]^T \yy_i - \max_{\yy_i' \not =\yy_i} [h(\xx_i)]^T \yy_i'.
\end{align*}
The hypothesis $h$ misclassifies the instance $(\xx_i, \yy_i)$ if $m_h(\xx_i, ~ \yy_i) \leq 0$.
If $0$-$1$ loss is used, we have $\ell(h(\xx_i), ~ \yy_i) = 1_{m_h(\xx_i, ~ \yy_i) \leq 0}.$
Because the $0$-$1$ loss is not continuous and thus hard to handle, we consider other loss functions that are continuous to upper bound this loss.
Specifically, we employ the hinge loss for multi-class classification:
$$\ell(h(\xx_i), ~ \yy_i) = |1 - m_h(\xx_i, ~ \yy_i)|_+.$$
The hinge loss is nondifferentiable when $m_h(\xx_i, ~ \yy_i) = 0,$ so we use the sub-gradient in this case.
For multi-class classification, the sub-gradient of the loss function is
\begin{equation}
 \label{eq.mc-loss-derivative}
 \frac{\partial ~\ell(h(\xx_i), ~ \yy_i)}{\partial~ \WW_t}=
 \left\{
 \begin{aligned}
 & \mathbf{0}_{D \times K}, ~~~~~~~~~~~~ m_h(\xx_i, ~ \yy_i) \geq 1, \\
 & \phi(\xx)[\yy_i'-\yy_i]^T, ~~~~~~~\text{otherwise},
 \end{aligned}
 \right.
\end{equation}
where $(\xx_i, ~ \yy_i)$ is an instance from the labeled sample $\mathcal{D}_l.$

\subsection{Multi-label Learning}
Consider the multi-label learning scenario with an output space $\mathcal{Y} = \{0, 1\}^K$ for multi-label classification and $\mathcal{Y} = \mathbb{R}^K$ for multi-label regression.
We define the loss function as
$
 \ell(h(\xx_i), ~ \yy_i) = \|\yy - h(\xx_i)\|_2^2.
$
The gradient of the loss is
\begin{equation}
 \label{eq.ml-loss-derivative}
 \frac{\partial ~\ell(h(\xx_i), ~ \yy_i)}{\partial~ \WW_t}=
 2 \phi(\xx_i) [h(\xx_i) - \yy_i]^T,
\end{equation}
where $(\xx_i, ~ \yy_i)$ is an instance from the labeled sample $\mathcal{D}_l.$

\bibliographystyle{elsarticle-num}
\bibliography{all}
\end{document}